%% file: main.tex
\documentclass[10pt,twocolumn,letterpaper]{article}

\usepackage[pagenumbers]{cvpr} %

\input{preamble}

\definecolor{cvprblue}{rgb}{0.21,0.49,0.74}
\usepackage[pagebackref,breaklinks,colorlinks,allcolors=cvprblue]{hyperref}

\def\paperID{*****} %
\def\confName{CVPR}
\def\confYear{2026}

\title{Flow Map Distillation Without Data}

\author{Shangyuan Tong\thanks{Equal contribution, \textsuperscript{$\dagger$}Equal advising}\\
MIT\\
{\tt\small sytong@csail.mit.edu}
\and
Nanye Ma\textsuperscript{*}\\
NYU\\
{\tt\small nm3607@nyu.edu}
\and
Saining Xie\textsuperscript{$\dagger$}\\
NYU\\
{\tt\small saining.xie@nyu.edu}
\and
Tommi Jaakkola\textsuperscript{$\dagger$}\\
MIT\\
{\tt\small tommi@csail.mit.edu}
}

\begin{document}
\maketitle

\begin{abstract}
State-of-the-art flow models achieve remarkable quality but require slow, iterative sampling. To accelerate this, flow maps can be distilled from pre-trained teachers, a procedure that conventionally requires sampling from an external dataset. We argue that this data-dependency introduces a fundamental risk of \textbf{Teacher-Data Mismatch}, as a static dataset may provide an incomplete or even misaligned representation of the teacher's full generative capabilities. This leads us to question whether this reliance on data is truly necessary for successful flow map distillation. In this work, we explore a data-free alternative that samples only from the prior distribution—a distribution the teacher is guaranteed to follow by construction, thereby circumventing the mismatch risk entirely. To demonstrate the practical viability of this philosophy, we introduce a principled framework that learns to predict the teacher's sampling path while actively correcting for its own compounding errors to ensure high fidelity. Our approach surpasses all data-based counterparts and establishes a new state-of-the-art by a significant margin. Specifically, distilling from SiT-XL/2+REPA, our method reaches an impressive FID of \textbf{1.45} on ImageNet 256${\times}$256, and \textbf{1.49} on ImageNet 512${\times}$512, both with only 1 sampling step. We hope our work establishes a more robust paradigm for accelerating generative models and motivates the broader adoption of flow map distillation without data.\\
{\small\faGlobe}\;Project page: \href{https://data-free-flow-distill.github.io/}{data-free-flow-distill.github.io}
\end{abstract}

\section{Introduction}
\label{sec:intro}

\begin{figure}[ht]
    \centering
    \includegraphics[width=0.85\linewidth, trim={0 3cm 0 0}, clip]{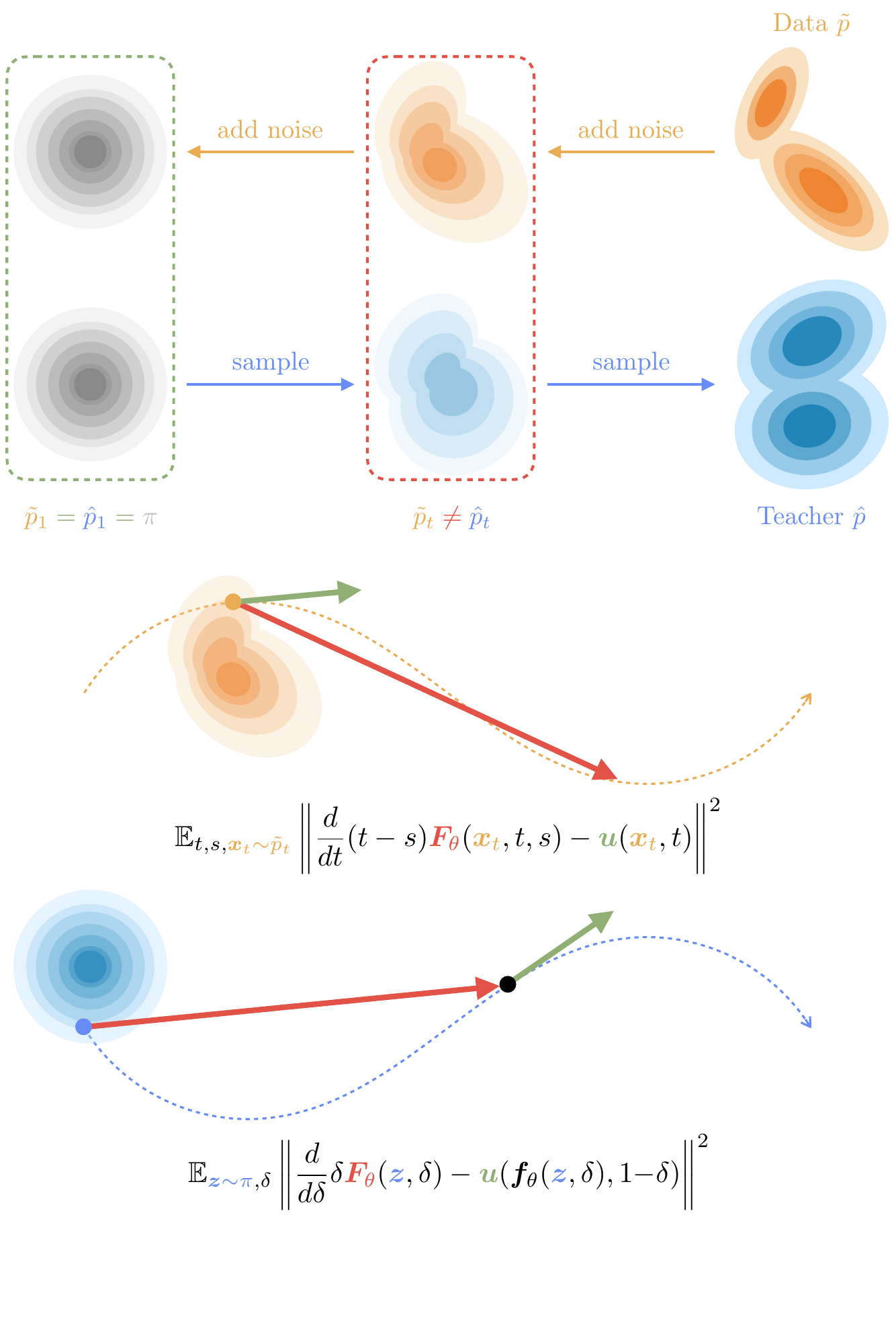}
    \caption{\textbf{Teacher-Data Mismatch and the data-free alternative.} (\ul{Top}) Conventional data-based distillation relies on intermediate distributions ({\color[HTML]{EAAC53} $\tilde{p}_t$}) derived from a static dataset, which could be misaligned with the teacher's true generative distributions ({\color[HTML]{688CF6} $\hat{p}_t$}). (\ul{Bottom}) The data-free paradigm, in contrast, samples only from the prior ($\pi$), the single distribution with guaranteed alignment, thereby circumventing the mismatch risk by construction.}
    \label{fig:teaser}
\end{figure}

Diffusion models~\citep{sohl2015deep,ho2020denoising,song2021scorebased,karras2022elucidating} and flow models~\citep{peluchetti2023non,liu2023flow,albergo2023building,lipman2023flow,heitz2023iterative,xu2022poisson,xu2023pfgm++} have revolutionized high-fidelity synthesis~\citep{ramesh2021zero,rombach2022high,watson2023novo,hoogeboom2022equivariant}, yet their reliance on numerically integrating an Ordinary Differential Equation (ODE) creates a significant computational bottleneck. To resolve this latency, flow maps~\citep{boffi2025flow}, which learn the solution operator of the ODE directly, offer a principled path to acceleration, bypassing iterative solving by taking large ``jumps'' along the generative trajectory. While flow maps can be trained from scratch~\citep{song2023consistency,frans2024one,geng2025mean,boffi2025build}, a more flexible alternative is to distill them from powerful, pre-trained teacher models~\citep{salimans2022progressive,berthelot2023tract,song2023consistency,kim2024consistency,sabour2025align}. This modular strategy allows for the compression of state-of-the-art models, which are often the product of advanced training~\citep{yu2024representation,leng2025repa,zheng2025diffusion} and post-training~\citep{xu2023imagereward,wallace2024diffusion,zhang2024large,liu2025flow,li2025mixgrpo} techniques.

We observe that the dominant and most successful flow map distillation approaches are \emph{data-based}, relying on samples from an external dataset to train the student. We argue that this tacitly accepted dependency introduces a fundamental risk of \textbf{Teacher-Data Mismatch}. As illustrated in \cref{fig:teaser}, a static dataset may provide an incomplete or misaligned representation of the teacher's true generative capabilities. This discrepancy arises frequently in practice: when a teacher generalizes beyond its original training set~\citep{song2025selective,kamb2024analytic,niedoba2024towards,scarvelis2023closed,pidstrigach2022score,yoon2023diffusion,yi2023generalization,kadkhodaie2023generalization}; when post-hoc fine-tuning~\citep{xu2023imagereward,wallace2024diffusion,zhang2024large,liu2025flow,li2025mixgrpo} shifts the teacher's distribution away from the original data; or when the teacher's proprietary training data is simply unavailable~\citep{rombach2022high,seedream2025seedream,labs2025flux,cao2025hunyuanimage,wu2025qwen}. In these scenarios, forcing a student to match the teacher on a misaligned dataset fundamentally constrains its potential.

Fortunately, this mismatch is not inevitable. We observe that while the teacher's generative paths may diverge from the dataset, they are, by definition, anchored to the prior distribution. As shown in \cref{fig:teaser}, the prior serves as the single point of guaranteed alignment: it is the shared origin for the teacher's generation and the termination point for any noising process. This insight leads us to question whether the common reliance on data is truly necessary. We posit that we can instead build a robust, \emph{data-free} alternative by sampling only from the prior, thereby circumventing the mismatch risk entirely by construction.

To operationalize this philosophy, we introduce a principled framework designed to track the teacher's dynamics purely from the prior. Our method takes a sample from the prior and a scalar integration interval, predicting where the flow should jump. We show that optimality is achieved when the model's \emph{generating velocity}, the rate at which it traverses its own path, aligns with the teacher's instantaneous velocity. Nevertheless, like any autonomous numerical solver, this prediction process is susceptible to compounding errors. To mitigate this, we propose a correction mechanism that further aligns the model's \emph{noising velocity}, the marginal velocity of the noising flow implied by the student's predicted distribution, back to the teacher. We name our proposal \textbf{FreeFlow}, emphasizing its defining characteristic as a completely data-free distillation framework for flow maps.

We validate our approach through extensive experiments on ImageNet~\citep{russakovsky2015imagenet}. Distilling from a SiT-XL/2+REPA~\citep{yu2024representation} teacher, FreeFlow establishes a new state-of-the-art, reaching an impressive FID of \textbf{1.45} on 256${\times}$256 and \textbf{1.49} on 512${\times}$512 with 1 function evaluation (1-NFE), significantly surpassing all data-based baselines. Furthermore, by leveraging its nature as a fast, consistent proxy, FreeFlow enables efficient inference-time scaling~\citep{ma2025inference,singhal2025general}, allowing for the search of optimal noise samples in a single step. Ultimately, our findings confirm that an external dataset is not an essential requirement for high-fidelity flow map distillation, and the risk of Teacher-Data Mismatch can be avoided entirely without compromising performance. We believe this work provides a more robust foundation for generative model acceleration and motivates a shift toward the data-free paradigm.

\section{Preliminaries}
\label{sec:background}

\paragraph{Diffusion and flow.}
Diffusion models~\citep{sohl2015deep,ho2020denoising,song2021scorebased,karras2022elucidating} and flow models~\citep{xu2022poisson,peluchetti2023non,liu2023flow,albergo2023building,lipman2023flow,heitz2023iterative,xu2023pfgm++} are trained to reverse a reference noising process that transports the data distribution $p\equiv p_0$ to a easy-to-sample prior distribution $\pi \equiv p_1$ like $\gN(\bm{0}, \mI)$. We denote the interpolating distributions in between as $p_t$ and their samples $\vx_t$, indexed by time $t\in[0,1]$. For the linear interpolation scheme~\citep{lipman2023flow,liu2023flow,ma2024sit} that we utilize throughout the paper, given a pair of terminal samples $\vx \sim p$ and $\vz\sim \pi$, we construct the noising process from its interpolants, $\vx_t = \vI_t(\vx, \vz) \coloneq (1-t)\vx+t\vz$, which in turn defines a conditional velocity, pointing in the direction from prior to data, $\vu(\vx_t,t\mid\vx,\vz) \coloneq -\partial_t \vI_t(\vx,\vz) = \vx-\vz$. Taking the expectation over $p$ and $\pi$, we arrive at the marginal instantaneous velocity, $\vu:\R^d\times[0,1]\to \R^d$, a vector field that dictates how the samples evolve, which governs the noising process with the following ODE:
\begin{align*}
    d \vx(t) = -\vu(\vx(t), t) dt,\numberthis\label{eq:ode}
\end{align*}
where $\vx(t)\in\R^d$ denotes the state of the system. In practice, the typically unknown $\vu$ can be well approximated by a model $\vg_\psi$ with parameters $\psi$, trained with denoising score matching~\citep{vincent2011connection,song2019generative} or conditional flow matching~\citep{peluchetti2023non,lipman2023flow}:
\begin{align*}
    \E_{\vx,\vz,t}\left\|\, \vg_\psi(\vI_t(\vx, \vz),t) - \vu(\vI_t(\vx, \vz),t\mid\vx,\vz) \,\right\|^2.\numberthis\label{eq:flow_loss}
\end{align*}
For sampling, we need to solve \cref{eq:ode} by integrating the flow backward in time:
\begin{align*}
    \vphi_\vu(\vx_t,t,s) = \vx_t + \int_t^{s} -\vu(\vx(\tau), \tau)d\tau,\numberthis\label{eq:solution}
\end{align*}
where $\vphi_\vu:D_\text{flow}\to\R^d, D_\text{flow}=\{(\vy,\zeta,\xi)\mid\vy\in\R^d,\zeta\in[0,1],\xi\in[0,\zeta]\}$, denotes the generating flow equipped with underlying velocity field $\vu$, and $t-s$ is the integration time interval. The standard sampling procedure of flow models corresponds to calculating $\vphi_\vu(\vz,1,0)$, $\vz\sim \pi$. In practice, we resort to some numerical solver, like Euler~\citep{song2020denoising,song2021scorebased} and Heun methods~\citep{karras2022elucidating}. Since the underlying trajectories typically exhibit complicated structure and curvature~\citep{xiao2022tackling,xu2023stable}, such numerical integration procedures often require dozens or even hundreds of NFEs for a single generation.

\paragraph{Flow maps.}
Instead of approximating the instantaneous velocity $\vu$, a flow map model~\citep{salimans2022progressive,berthelot2023tract,song2023consistency,kim2024consistency,frans2024one,geng2025mean,sabour2025align,boffi2025flow}, $\vf_\theta$ parameterized by $\theta$, is trained to directly approximate $\vphi_\vu$. Existing works dissect and utilize key properties of $\vphi_\vu$ to construct their training objective, typically via the local dynamics described by $\vu$. For example, in MeanFlow~\citep{geng2025mean}, the network $\vF_\theta:D_\text{flow}\to\R^d$ represents the average velocity $\vf$ travels over its path: $\vf_\theta(\vx_t,t,s) = \vx_t + (t-s) \vF_\theta(\vx_t,t,s)$. At optimality, we know from \cref{eq:solution} that $(t-s)\vF_{\theta^*}(\vx_t,t,s) = \int_t^{s}-\vu\left(\vx(\tau),\tau\right)d\tau$. Differentiating both sides \wrt $t$ leads to
\begin{align*}
    \vF_{\theta^*}(\vx_t,t,s) + (t-s)\frac{d}{dt}\vF_{\theta^*}(\vx_t,t,s) = \vu(\vx_t,t),\numberthis\label{eq:mf_F_identity}
\end{align*}
where $\frac{d}{dt}\vF$ is the total derivative that can be further expanded into $\nabla_{\vx_t}\vF\frac{d\vx_t}{dt}+\partial_t\vF$. This identity then motivates the following practical training objective:
\begin{align*}
    \E_{t,s,\vx_t} \left\|\, \vF_\theta(\vx_t,t,s) - \sg\left(\vu_\text{MF}\right) \,\right\|^2,\numberthis\label{eq:mf_loss}
\end{align*}
where $\sg(\cdot)$ denotes the stop-gradient operation, and $\vu_\text{MF}=\vu(\vx_t,t) - (t-s)\frac{d}{dt}\vF_{\theta}(\vx_t,t,s)$. The decision to drop the remaining gradients is mostly empirical and common in prior literature~\citep{song2024improved,lu2024simplifying,frans2024one,geng2025mean}, as it results in faster, less resource-demanding, and often more stable training. Here, $\vu$ is either co-trained from scratch~\citep{geng2025mean} together with \cref{eq:flow_loss}, or a pre-trained flow model~\citep{song2023consistency,sabour2025align,peng2025flow} in a procedure known as flow map distillation. In this paper, we focus on the second case, since this modular approach allows for easier incorporation of advanced training~\citep{yu2024representation,leng2025repa,zheng2025diffusion} and post-training~\citep{xu2023imagereward,wallace2024diffusion,zhang2024large,liu2025flow,li2025mixgrpo} techniques.

\begin{figure}[t!]
    \centering
    \includegraphics[width=0.9\linewidth]{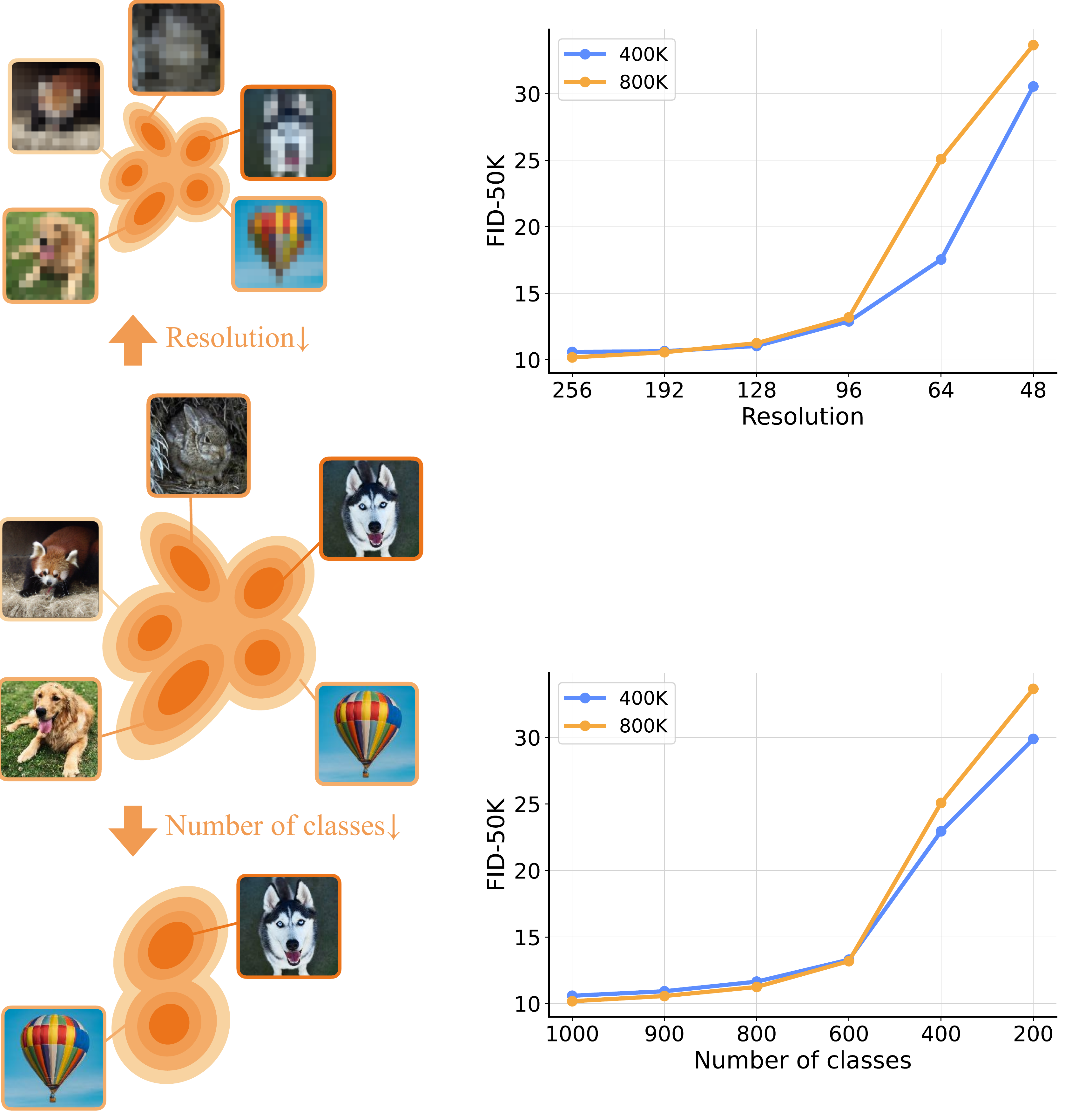}
    \caption{\textbf{Impact of Teacher-Data Mismatch.} With a fixed teacher model, increasing augmentation induces a more severe mismatch between teacher and data, degrading student performance.}
    \label{fig:mismatch}
    \vspace{-0.3cm}
\end{figure}

\begin{figure*}[t]
    \centering
    \includegraphics[width=0.8\linewidth]{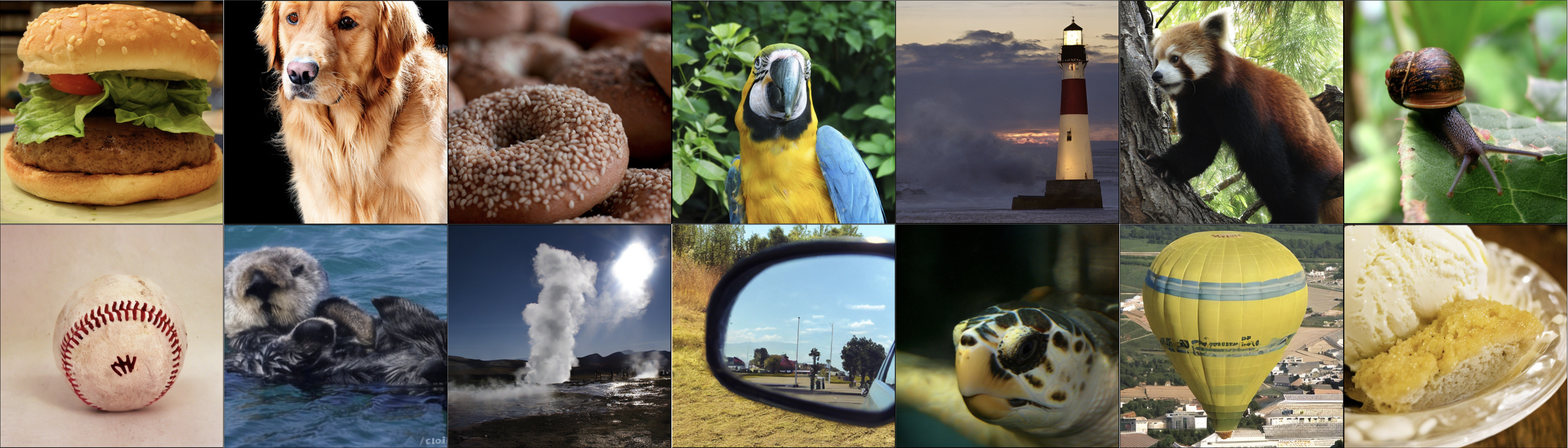}
    \caption{\textbf{Selected samples from FreeFlow-XL/2 model at 512${\times}$512 resolution with 1-NFE.} More uncurated results are in \cref{sec:visual}.}
    \vspace{-0.3cm}
\end{figure*}

\section{With or Without Data?}
\label{sec:motivation}

The goal of flow map distillation is to create a student $\vf_\theta$ that faithfully reproduces the full generative process of the given $\vphi_\vu$, just with fewer NFEs. Intuitively, existing methods learn the teacher sampling dynamics at a series of intermediate states $\vx_t$ (note the expectation over $\vx_t$ in \cref{eq:mf_loss}). Our discussion begins from this very point, with a closer inspection of a foundational, yet largely unexamined, element: the underlying distribution from which these states $\vx_t$ are drawn.

\subsection{The Risk of Teacher-Data Mismatch}
\label{sec:mismatch}

The distribution of $\vx_t$ is conventionally formed by sampling from a \emph{data-noising} distribution, which we denote as $\tilde{p}_t$. This is the set of all interpolants $\vI_t(\tilde{\vx}, \vz)$ generated by taking a data point $\tilde{\vx} \sim \tilde{p}$\footnote{We use $\tilde{p}$ to denote the dataset available at distillation time. As we will discuss, it may differ from $p$.} and a prior sample $\vz \sim \pi$. Although tacitly accepted, this practice implicitly assumes that $\tilde{p}_t$ is a suitable representation of the states the teacher model follows over its sampling trajectory.

We note that the teacher defines a distinct set of intermediate states via \cref{eq:solution}. The set of all $\vx_t$ along these solution paths constitutes the true \emph{teacher-generating} distribution, which we denote as $\hat{p}_t$. For the student to perfectly reproduce the teacher, it should be trained to match the teacher's dynamics over $\hat{p}_t$. The central problem that we identify in this paper, which we term the \textbf{Teacher-Data Mismatch}, is that these two distributions are not equivalent: $\tilde{p}_t \neq \hat{p}_t$.

By training on $\tilde{p}_t$, the student is compelled to learn the teacher's dynamics only on trajectories that are anchored to the static dataset $\tilde{p}$. Any generative behavior of the teacher that starts from $\pi$ and evolves through states not well-represented by $\tilde{p}_t$ will be systematically ignored during training. Consequently, even a perfectly converged student is not guaranteed to reproduce the teacher's outputs, as it has fundamentally been trained to distill the wrong process.

To examine and validate the impact of the discussed mismatch, we design a controlled experiment on ImageNet, where we introduce deliberately misaligned $\tilde{p}_t$ distributions by applying data augmentations during the training of conventional flow map distillation. As shown in \cref{fig:mismatch}, the quality of the learned flow map is highly sensitive to the fidelity and representativeness of the distillation dataset: stronger augmentation leads to a larger discrepancy between $\hat{p}_t$ and $\tilde{p}_t$, and, in turn, results in a more significant degradation in student performance.

This mismatch is not merely a theoretical curiosity; it manifests in several common and critical scenarios. First, when a powerful teacher model has generalized beyond its training set~\citep{song2025selective,kamb2024analytic,niedoba2024towards,scarvelis2023closed,pidstrigach2022score,yoon2023diffusion,yi2023generalization,kadkhodaie2023generalization}, or even when it simply employs the widely adopted Classifier-Free Guidance (CFG)~\citep{ho2022classifier}, its generative distribution $\hat{p}_0$ will contain novel, extrapolated samples not present in $p$, causing its trajectories $\hat{p}_t$ to necessarily diverge from the data-noising paths $p_t$. Second, if the teacher has been altered by post-hoc fine-tuning~\citep{xu2023imagereward,wallace2024diffusion,zhang2024large,liu2025flow,li2025mixgrpo}, its generating flow is deliberately modified, again forcing $\hat{p}_t$ to diverge from the original data-noising distribution. Third, a teacher model may be released publicly while its massive, proprietary training data is not~\citep{rombach2022high,seedream2025seedream,labs2025flux,cao2025hunyuanimage,wu2025qwen}. In this case, $p$ is simply unavailable, and any proxy dataset used will almost certainly create a severe mismatch.

\subsection{Towards a Data-Free Alternative}
\label{sec:data_free}

A straightforward remedy to the Teacher-Data Mismatch is to directly sample from $\hat{p}_t$ during training. This would involve sampling $\vz \sim \pi$, integrating the teacher model from $t=1$ to a random time $t$ to get $\vx_t = \vphi_\vu(\vz, 1, t)$, and then using this $\vx_t$ in the distillation loss. Just like the case of knowledge distillation~\citep{hinton2015distilling,luhman2021knowledge,zheng2023fast,liu2023flow}, where the model is trained to learn the fully integrated outcomes at $t=0$, obtaining reference trajectories on-the-fly is prohibitively costly, whereas pre-computing them offline scales poorly. In short, for high-dimensional or complex conditional tasks, generating enough samples to adequately represent the underlying distribution simply becomes intractable.

This apparent impasse leads us to re-examine the properties of these two distributions. While $\hat{p}_t$ and $\tilde{p}_t$ diverge for $t \in [0, 1)$, they are, by construction, identical at $t=1$. The data-noising process terminates at the prior distribution (\ie, $\tilde{p}_1 \equiv \pi$), and the teacher's generative process begins at the same prior (\ie, $\hat{p}_1 \equiv \pi$). This observation provides a crucial foothold. The prior $\pi$ is the one distribution we can sample from that is guaranteed to be on-distribution for the teacher's generative process, completely circumventing the risk of Teacher-Data Mismatch.

This insight motivates our central question: Is the commonly followed data-dependency truly necessary for flow map distillation? We argue that it is not. In the following, we explore a data-free alternative, building a new flow map distillation objective governed only by the prior distribution.

\section{Flow Map Distillation Without Data}
\label{sec:method}

Our exploration now moves from motivation to mechanism. A flow map model can be generally understood as directly modeling a segment of a full generative trajectory, and the core principle for training such a model is to enforce consistency with $\vu$ at some point along this segment, ensuring the learned dynamics are locally correct. The segment's two key points provide natural candidates: a sampled start-point, $\vx_t$, and a predicted end-point, $\vx_s = \vf_\theta(\vx_t, t, s)$.

This perspective provides a clear lens through which to view the distillation process. In the conventional, data-based setting, the start-point $\vx_t$ is drawn from a series of data-noising distributions $\tilde{p}_t$. It is thus natural to constrain the model by perturbing this start-point, which corresponds to differentiating the optimal condition, $(t-s)\vF_{\theta^*}(\vx_t,t,s) = \int_t^{s}-\vu\left(\vx(\tau),\tau\right)d\tau$, with respect to $t$. This operation leads to the MeanFlow identity~\citep{geng2025mean} in \cref{eq:mf_F_identity}, which effectively enforces consistency at the start of the segment.

In our data-free investigation, however, we only sample our start-point from the prior $\pi$, which fixes $\vx_t = \vz$ at $t=1$. Consequently, perturbing the start-time $t$, and in turn the start-point, is no longer a meaningful operation. Thus, we consider the symmetrical alternative: if we cannot enforce consistency by perturbing the sampled start-point, we can instead do so by perturbing the predicted end-point. This provides a different path to ensuring the student's local dynamics are correct, and it corresponds to differentiating the optimal condition with respect to the end time $s$.

To formalize this, we first simplify our notation to reflect this prior-anchored ($t=1$) view. That is: (1) We define the integration duration as $\delta = t-s = 1-s$, where $\delta \in [0, 1]$; (2) The flow map $\vf_\theta(\vz, \delta): \R^d\times[0,1]\to\R^d$ approximates the true flow $\vphi_\vu(\vz, 1, 1-\delta)$; (3) The average velocity $\vF_\theta(\vz, \delta): \R^d\times[0,1]\to\R^d$ is linked by the parameterization $\vf_\theta(\vz, \delta) = \vz + \delta\vF_\theta(\vz, \delta)$. The optimal condition, anchored at $t=1$, thus reduced to:
\begin{align*}
    \delta\vF_{\theta^*}(\vz,\delta) = \int_1^{1-\delta}-\vu\left(\vx(\tau),\tau\right)d\tau.\numberthis\label{eq:F_optimal}
\end{align*}
Following our exposition, we differentiate both sides of \cref{eq:F_optimal} \wrt $\delta$ (equivalent to differentiating \wrt $-s$):
\begin{align*}
    \vF_{\theta^*}(\vz,\delta) + \delta\partial_\delta\vF_{\theta^*}(\vz,\delta) = \vu\left(\vf_{\theta^*}(\vz,\delta),1-\delta\right).\numberthis\label{eq:F_identity}
\end{align*}
\cref{eq:F_identity} differs from \cref{eq:mf_F_identity} in subtle ways: (1) The time derivative of $\vF_\theta$ is just a partial derivative, as $\vz$ does not depend on $\delta$; (2) $\vu$ is evaluated at a state predicted by $\vf_\theta$.

The identity defined in \cref{eq:F_identity} provides a sufficient condition for optimality, which motivates the following loss:
\begin{align*}
    \E_{\vz,\delta}\left\|\, \vF_\theta(\vz,\delta) - \sg( \vu_\text{target}) \,\right\|^2,\numberthis\label{eq:p_loss_c}
\end{align*}
where $\vu_\text{target} = \vu\left(\vf_{\theta}(\vz,\delta),1-\delta\right) - \delta\partial_\delta \vF_{\theta}(\vz,\delta)$. Remarkably, we verify that \cref{eq:p_loss_c} is formulated by only sampling from the prior $\pi$, without any reliance on an external dataset $\tilde{p}$ and thus free from the risks of Teacher-Data Mismatch. Hence, it achieves the goal of our exploration.

\subsection{Predict With Generating Flows}
\label{sec:predict}

We now analyze our proposed objective in \cref{eq:p_loss_c}. Note that $\partial_\delta \vf_\theta(\vz,\delta)$, the model prediction's rate of change with respect to the integration time, is the velocity with which the model travels along its generating flow. Thus, the optimality of the student is equivalent to the alignment between the model's \emph{generating velocity} and the underlying velocity. Indeed, it is easy to see that the loss value of \cref{eq:p_loss_c} is the same as $\E_{\vz,\delta}\| \partial_\delta \vf_{\theta}(\vz,\delta) - \vu(\vf_{\theta}(\vz,\delta),1-\delta) \|^2$, which evaluates to 0 if and only if $\partial_\delta \vf = \vu$. Intuitively, the student is analogous to an autonomous ODE solver, which uses its current estimated state to query the derivative function and compute the next state. The student learns to ``ride'' the teacher's vector field, starting from $\pi$ and extending outward, step by step, based entirely on its own evolving predictions.

In practice, $\partial_\delta \vF_\theta$ can be calculated easily and efficiently via Jacobian-vector product (JVP) with forward-mode automatic differentiation, barring some advanced computation kernels, which currently require customized solutions~\citep{lu2024simplifying}. Still, it is desirable to work with a more flexible loss function with no such limitations, which is why we derive a discrete-time alternative (detail in \cref{sec:ext_pred}) that numerically approximates $\partial_\delta \vF_\theta$ with finite differences. 

Consequently, we abstract away the computation detail, and use the general notation $\vvG(\vf_\theta(\vz,\delta), 1-\delta)$ to denote the student's generating velocity $\partial_\delta \vf_\theta(\vz,\delta)$. The understanding that \cref{eq:p_loss_c} aligns $\vvG$ and $\vu$ can also be observed from its optimization gradients:
\begin{align*}
    \nabla_\theta\, \E_{\vz,\delta}\biggl[ \vF_\theta(\vz,\delta)^\top \sg\Bigl( \Delta_{\vvG,\vu}(\vf_{\theta}(\vz,\delta),1-\delta) \Bigr) \biggr],\numberthis\label{eq:p_grad}
\end{align*}
where $\Delta_{\vvG,\vu}(\vf_{\theta}(\vz,\delta),1-\delta)$ is the difference between $\vvG(\vf_{\theta}(\vz,\delta),1-\delta)$ and $\vu(\vf_{\theta}(\vz,\delta),1-\delta)$. Explicitly writing out the gradients makes it easier to adopt techniques like gradient weighting/normalization explored in \cref{sec:design}. Note that, if $\vF_\theta$ is further parameterized, we should replace the first term in \cref{eq:p_grad} with the actual network output to ensure effective gradient control by only modifying $\Delta_{\vvG,\vu}$.

Lastly, we note that advanced sampling techniques can be easily incorporated in $\vu$, \eg, for classifier-free guidance (CFG)~\citep{ho2022classifier}, we could simply replace $\vu(\vx_t,t\mid c)$\footnote{We omit $c$ everywhere else in the paper for simplicity.} with $\vu_\gamma(\vx_t,t\mid c) = \gamma \cdot \vu(\vx_t,t\mid c) + (1 - \gamma) \cdot \vu(\vx_t, t\mid c=\emptyset)$, where $c$ is condition with $\emptyset$ referring to a null input, and $\gamma$ is the guidance strength. Furthermore, the model can be trained on a range of $\gamma$ values, which enables the ability to effortlessly change the guidance strength at inference time.

\paragraph{The Challenge of Error Accumulation.} In practice, the student model $\vf_\theta$ is only a learned approximation, not a mathematically perfect one. At any $\delta$, its prediction $\vf_\theta(\vz, \delta)$ may contain a small approximation error, placing it slightly off the true teacher trajectory $\vphi_\vu(\vz, 1, 1-\delta)$. Because the objective is self-referential, this small deviation influences the target used for subsequent steps. The student queries the teacher at its current and potentially slightly erroneous state, and the resulting velocity target, while correct for that state\footnote{This assumes the teacher is perfect, which is not true in practice.}, may not guide the student back toward the true path. Such errors can compound as the integration proceeds from $\delta=0$ to $\delta=1$. In \cref{fig:err_acc}, we measure the relative differences between the student's predicted trajectory and the teacher's true sampling path, which empirically quantifies and confirms such a phenomenon as the student progressively diverges from the teacher when $\delta$ increases.

\begin{figure}[t]
    \centering
    \begin{minipage}[t]{0.47\linewidth}
        \vspace{0pt} %
        \centering
        \includegraphics[width=\linewidth]{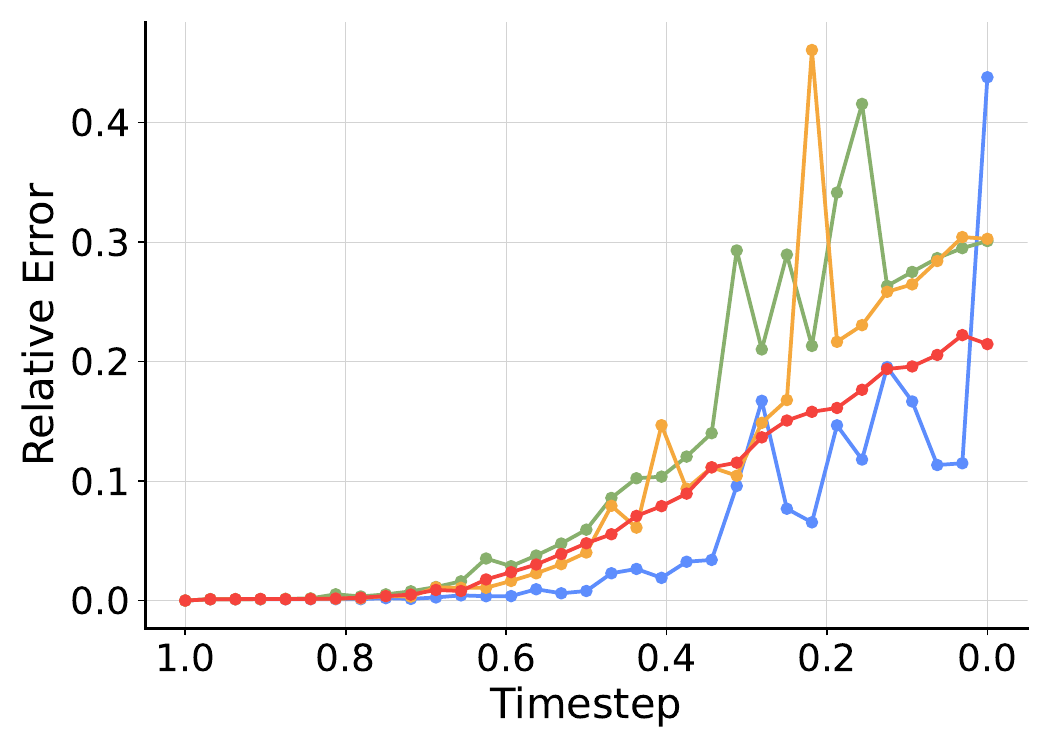}
    \end{minipage}
    \hfill
    \begin{minipage}[t]{0.47\linewidth}
        \vspace{0pt} %
        \centering
        \includegraphics[width=0.88\linewidth]{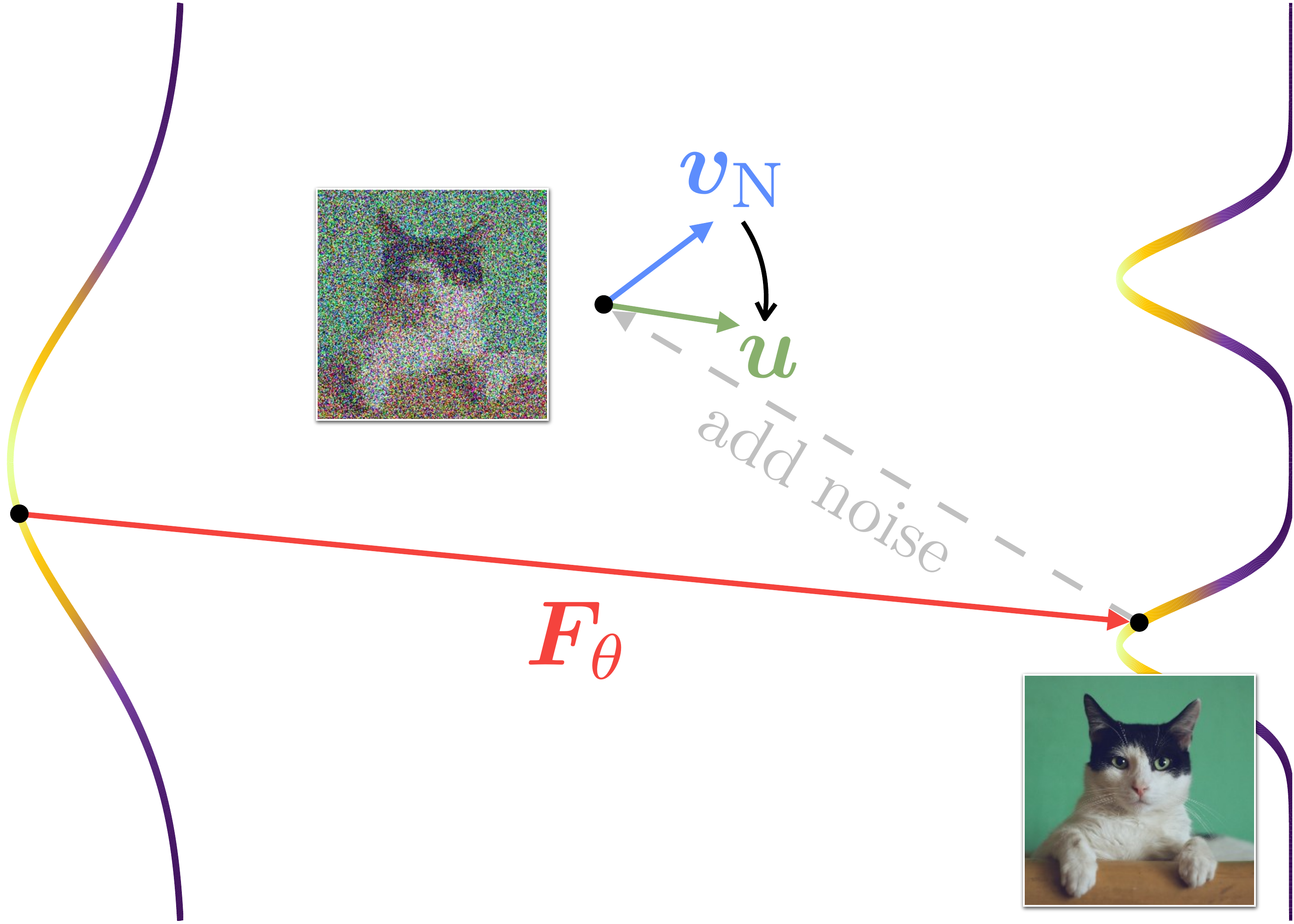}
    \end{minipage}
    \vspace{\abovecaptionskip}
    \begin{minipage}[t]{0.47\linewidth}
        \caption{Approximation errors accumulate as the prediction proceeds from noise to data.}
        \label{fig:err_acc}
    \end{minipage}
    \hfill
    \begin{minipage}[t]{0.47\linewidth}
        \caption{Correction objective in \cref{eq:c_grad} aligns the student's noising velocity $\vvN$ with $\vu$.}
        \label{fig:corr}
    \end{minipage}
    \vspace{-0.5cm}
\end{figure}

\subsection{Correct With Noising Flows}
\label{sec:correct}

The problem identified above is that the student is trained to predict the next state based on its current one, but it has no means to correct its own deviations and pull the trajectory back towards the teacher's true path. Drawing inspiration from \citet{song2021scorebased}, we seek to correct the marginal distributions of the student solutions, analogous to a predictor-corrector method for solving ODEs~\citep{allgower2012numerical}. Additionally, the correction objective cannot reintroduce the data-dependency we have worked to remove, meaning that we do not consider objectives like GANs~\citep{goodfellow2014generative} that rely on an external dataset.

Variational Score Distillation~\citep{wang2023prolificdreamer} was originally proposed as a training procedure to distill distributions from pre-trained diffusion models by minimizing the Integral KL divergence~\citep{luo2023diff}. We slightly adapt it to our setting, where we use $q\equiv q_0$ to denote the marginal distribution of clean samples generated by the model, $\int \vf_\theta(\vz, 1) d\pi$. Specifically, it has been shown that $q=p$ if and only if their IKL divergence is 0, which is defined as
\begin{align*}
    D_\text{IKL}(q \parallel p) \coloneq \int_0^1 \E_{\vx_r\sim q_r}\left[\log\frac{q_r(\vx_r)}{p_r(\vx_r)}\right] dr,\numberthis\label{eq:ikl}
\end{align*}
where $q_r$ and $p_r$ are the marginal interpolating distributions, following the same noising process constructed by $\vI$.

The optimization gradient of \cref{eq:ikl} \wrt $\theta$ is $\E_{r,\vx_r}\left[ \left(\nabla_\theta \vx_r \right)^\top \left(\nabla_{\vx_r} \log q_r(\vx_r) - \nabla_{\vx_r} \log p_r(\vx_r)  \right) \right]$. As the score functions are interchangeable with the marginal velocities~\citep{ma2024sit}, we can optimize with the following gradient instead for correcting the student's prediction:
\begin{align*}
    \nabla_\theta\, \E_{\vz,\vn,r} \biggl[ \vF_\theta(\vz,1)^\top \sg\Bigl( \Delta_{\vvN,\vu}(\vI_r(\vf_{\theta}(\vz,1), \vn), r) \Bigr) \biggr],\numberthis\label{eq:c_grad}
\end{align*}
where $\vn$ is sampled from the prior $\pi$ like $\vz$. We verify that \cref{eq:c_grad} is also formulated by only sampling from $\pi$, free from the risks of Teacher-Data Mismatch. Here, $\vvN$ denotes the marginal velocity of the noising flow constructed from the generated distribution $q$ with the interpolating function $\vI$, and $\Delta_{\vvN,\vu}$ is the difference between $\vvN$ and $\vu$. We illustrate the high-level understanding of this mechanism in \cref{fig:corr}.

Since $\vvN$ is unknown, we approximate it with another online network $\vg_\psi$\footnote{Further parameterizations on $\vg_\psi$ are permissible.}, full-parameter~\citep{luo2023diff,yin2024one,yin2024improved,zhou2024score,zhou2024adversarial,xu2025one} or LoRA~\citep{hu2022lora,wang2023prolificdreamer,nguyen2024swiftbrush}, with loss in \cref{eq:flow_loss}. More specifically, for a pair of samples $\vf_{\theta}(\vz,1)$ and $\vn$, the conditional noising velocity is $-\partial_r\vI_r(\vf_{\theta}(\vz,1), \vn)$, and we arrive at $\vvN$ by taking the expectation over $\vz\sim\pi$ and $\vn\sim\pi$:
\begin{align*}
    \E_{\vz,\vn,r}\left\|\, \vg_\psi(\vI_r(\vf_{\theta}(\vz,1), \vn), r) + \partial_r\vI_r(\vf_{\theta}(\vz,1), \vn) \,\right\|^2.\numberthis\label{eq:c_psi_loss}
\end{align*}

We highlight the similarity of the gradient forms between \cref{eq:p_grad} and \cref{eq:c_grad}. Consequently, we identify that the optimality of the student is also equivalent to the alignment between the model's \emph{noising velocity} and the underlying velocity. That is, \cref{eq:ikl} evaluates to 0 if and only if $\Delta_{\vvN,\vu}=\bm{0}$. Such a velocity alignment perspective offers a series of new understandings, which provide the essential reasoning behind the practical design choices discussed later in \cref{sec:design}. We further note that a comprehensive correction procedure should correct the full predicted trajectory of the student, rather than only the end sample considered in \cref{eq:ikl}. However, we do not find such a design helpful in the experiment settings we considered in \cref{sec:exp}.

\section{Experiments}
\label{sec:exp}

We empirically validate our proposed method on ImageNet~\citep{russakovsky2015imagenet} at 256${\times}$256 and 512${\times}$512 resolutions using FID-50K~\citep{heusel2017gans}, with implementation details provided in \cref{sec:implement}.

\begin{figure*}[t]
    \centering 
    \begin{minipage}[t]{0.32\linewidth}
        \centering
        \includegraphics[width=0.9\linewidth]{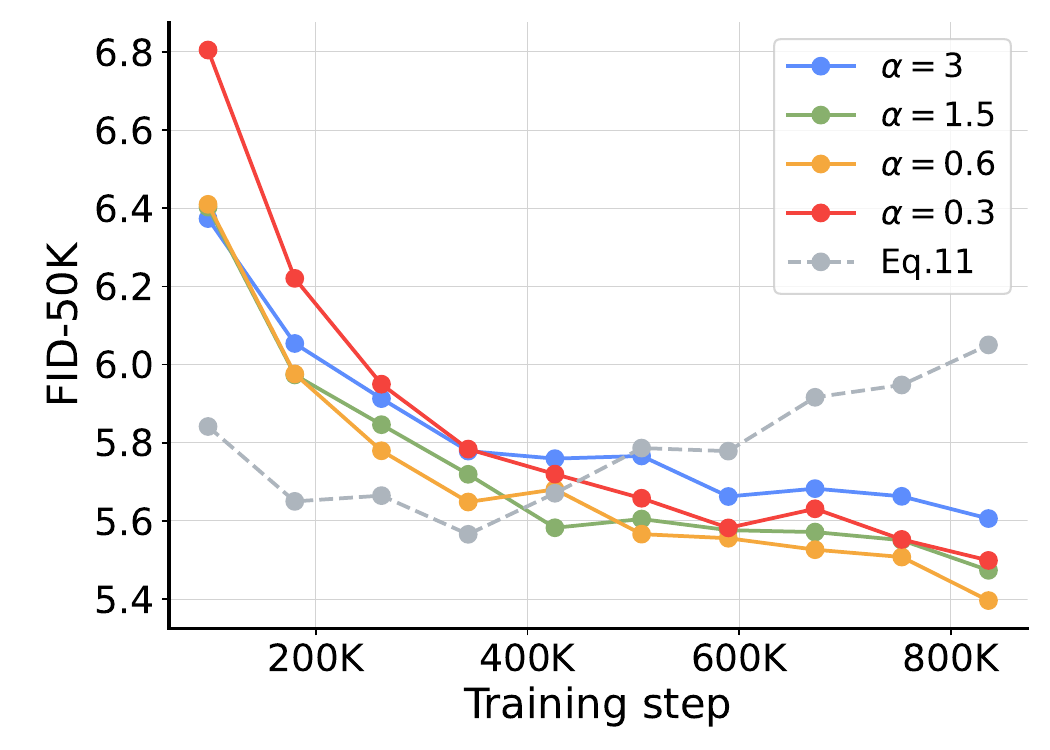} 
        \caption{Performance is robust across $\alpha$.}
        \label{fig:alpha}
    \end{minipage}
    \hfill
    \begin{minipage}[t]{0.32\linewidth}
        \centering
        \includegraphics[width=0.9\linewidth]{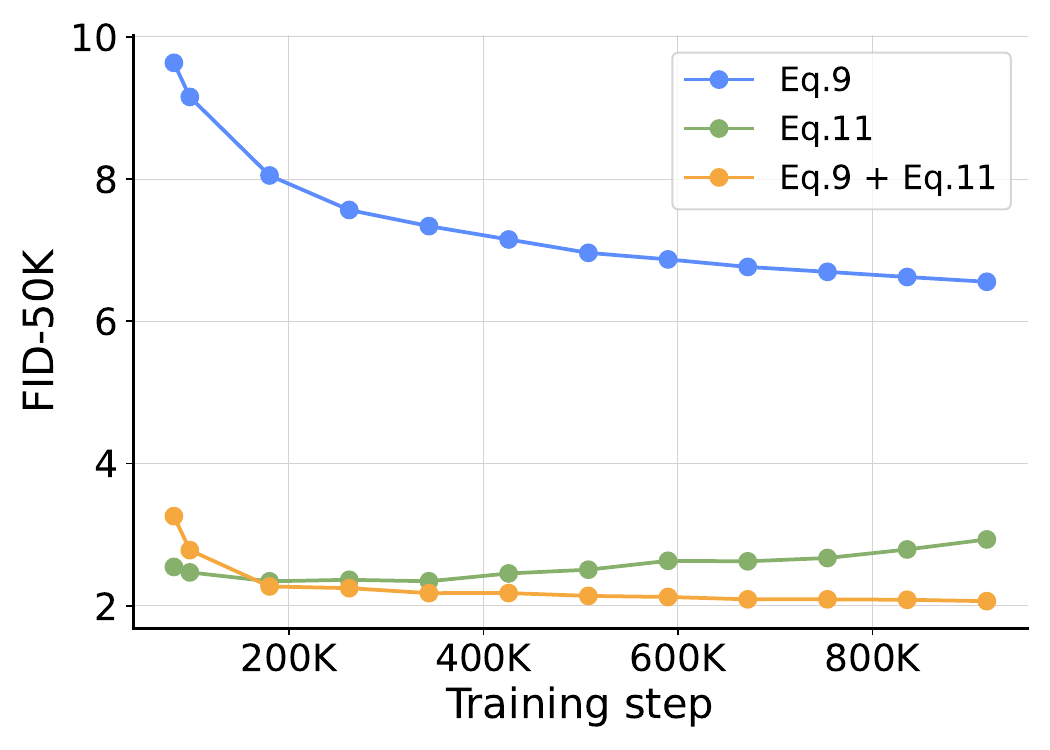}
        \caption{Synergy between \cref{eq:p_grad,eq:c_grad}.}
        \label{fig:synergy}
    \end{minipage}
    \hfill
    \begin{minipage}[t]{0.32\linewidth}
        \centering
        \includegraphics[width=0.9\linewidth]{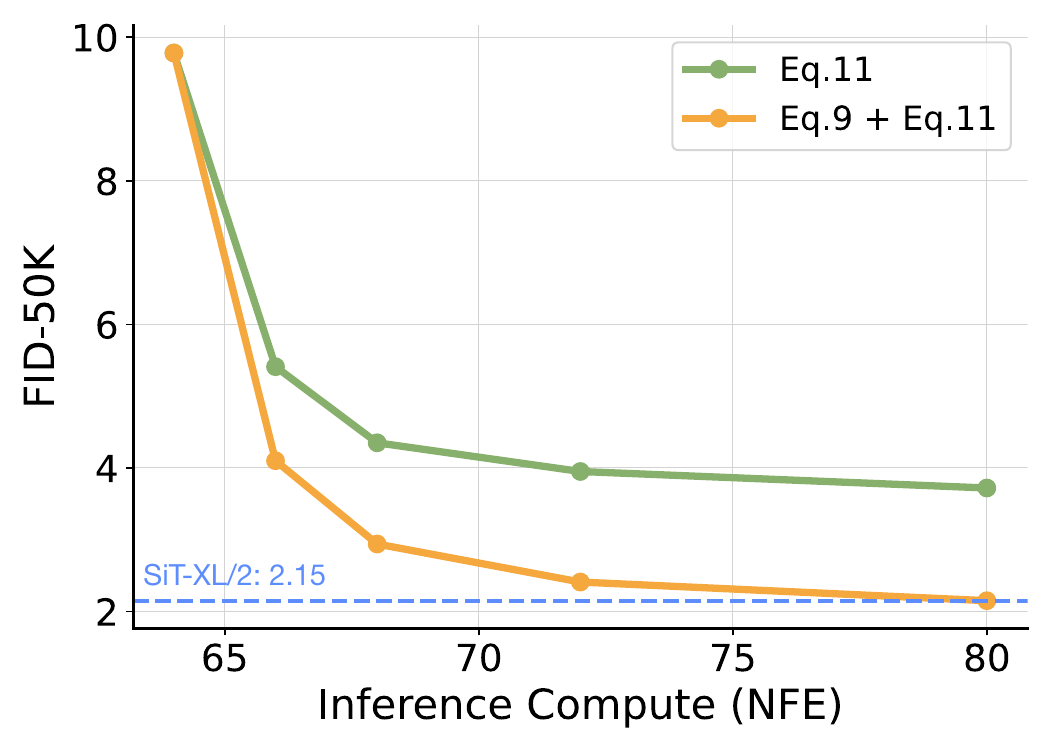}
        \caption{Inference-time scaling.}
        \label{fig:tts}
    \end{minipage}
    \vspace{-0.3cm}
\end{figure*}

\subsection{Design Decisions}
\label{sec:design}

We analyze each design choice through targeted qualitative and quantitative studies, presenting key findings in the main text and deferring full analyses to the \cref{sec:ext_discuss}. Unless specified otherwise, we adopt the DiT-B/2 architecture~\citep{peebles2023scalable}, use the pre-trained SiT-B/2~\citep{ma2024sit} as $\vu$ and student initialization, train the model for 400K iterations (roughly equivalent to 80 epochs with a batch size of 256) with uniformly sampling $\gamma\in[1,2]$, and evaluate with the best $\gamma=2$. We begin with designs specific to training with one of \cref{eq:p_grad,eq:c_grad}, followed by studies on how to properly combine the two.

\begin{table}[t]
    \centering
    \scalebox{0.98}{
    \begin{subtable}[t]{0.48\columnwidth}
        \tablestyle{6pt}{1.1}
        \begin{tabular}[t]{lS[table-format=1.2, mode=text]}
            \toprule
            $r$ sampling; \cref{eq:c_grad} & \multicolumn{1}{c}{FID $\downarrow$} \\
            \midrule
            LogitNormal(-0.4, 1.6) & 6.24 \\
            LogitNormal(0.0, 1.6) & 5.95 \\
            LogitNormal(0.4, 1.6) & 5.78 \\
            LogitNormal(0.8, 1.6) & 5.63 \\
            LogitNormal(1.2, 1.6) & 5.78 \\
            \bottomrule
        \end{tabular}
        \caption{\textbf{$r$ sampling.} More emphasis on higher noise levels leads to better results when aligning $\vvN$ and $\vu$.}
        \label{tab:r_sample}
    \end{subtable}
    }
    \hfill
    \scalebox{0.98}{
    \begin{subtable}[t]{0.48\columnwidth}
        \centering
        \tablestyle{6pt}{1.1}
        \begin{tabular}[t]{lS[table-format=2.2, mode=text]} 
            \toprule
            $r$ range; \cref{eq:c_grad} & \multicolumn{1}{c}{FID $\downarrow$} \\
            \midrule
            {[0, 0.6]} & 91.82 \\
            {[0, 0.7]} & 24.62 \\
            {[0, 0.8]} & 9.00 \\
            {[0, 0.9]} & 6.64 \\
            {[0, 1.0]} & 6.02 \\
            \bottomrule
        \end{tabular}
        \caption{\textbf{$r$ range.} With a uniform distribution over $r$, dropping higher noise levels leads to worse results.}
        \label{tab:r_range}
    \end{subtable}
    }

    \vspace{1ex}

    \scalebox{0.98}{
    \begin{subtable}[t]{0.48\columnwidth}
        \centering
        \tablestyle{6pt}{1.1}
        \begin{tabular}[t]{lS[table-format=1.2, mode=text]} 
            \toprule
            interval; \cref{eq:c_grad} & \multicolumn{1}{c}{FID $\downarrow$} \\
            \midrule
            {[0, 0.5]} & 7.09 \\
            {[0, 0.6]} & 5.72 \\
            {[0, 0.7]} & 5.63 \\
            {[0, 0.8]} & 6.44 \\
            {[0, 0.9]} & 7.41 \\
            {[0, 1.0]} & 8.65 \\
            \bottomrule
        \end{tabular}
        \caption{\textbf{Guidance interval.} Compared to teacher sampling, a more aggressive guidance interval is better.}
        \label{tab:guid_intvl}
    \end{subtable}
    }
    \hfill
    \scalebox{0.98}{
    \begin{subtable}[t]{0.48\columnwidth}
        \centering
        \tablestyle{6pt}{1.1}
        \begin{tabular}[t]{lS[table-format=1.1, mode=text]S[table-format=2.2, mode=text]} 
            \toprule
            objective & \multicolumn{1}{c}{$k$} & \multicolumn{1}{c}{FID $\downarrow$} \\
            \midrule
            \multirow{3}{*}{\cref{eq:p_grad}} & 0.0 & 11.91 \\
            & 0.5 & 11.71 \\
            & 1.0 & 12.40 \\
            \arrayrulecolor{black!30}\specialrule{\lightrulewidth}{0pt}{0pt}\arrayrulecolor{black}
            \multirow{3}{*}{\cref{eq:p_grad,eq:c_grad}} & 0.0 & 43.53 \\
            & 0.5 & 10.58 \\
            & 1.0 & 5.58 \\
            \bottomrule
        \end{tabular}
        \caption{\textbf{Gradient weighting.} With both objectives, stronger decay on $\Delta_{\vvG,\vu}$ leads to better training.}
        \label{tab:k_impact}
    \end{subtable}
    }
    \caption{Empirical investigations of various design decisions.}
    \label{tab:design}
    \vspace{-0.5cm}
\end{table}

\paragraph{Sampling of $r$ in \cref{eq:c_grad}.}
Traditionally, within the diffusion framework and from the divergence minimization perspective, the sampling of $r$ in \cref{eq:c_grad} often follows a uniform distribution over the discretized steps designed for generation~\citep{poole2022dreamfusion,wang2023prolificdreamer,yin2024one}. Here, we provide a new perspective on the \cref{eq:c_grad}, which drives our proposal on the sampling distribution of $r$. Recall that the optimality of our correction objective is $\Delta_{\vvN,\vu}=\bm{0}$, the alignment between the noising velocity of the model's generated distribution and the underlying velocity. The velocity fields induce a pair of continuity equations $\partial_t p_t(\vx) = - \nabla_\vx \cdot (p_t(\vx) \vu(\vx, t))$ and $\partial_t q_t(\vx) = - \nabla_\vx \cdot (q_t(\vx) \vvN(\vx, t))$, which dictate the evolution of $p$ and $q$. We note that $p_1=q_1=\pi$ by construction, and the gap between $p_0$ and $q_0$ can be understood as the time-integrated accumulation of the differences in their corresponding probability fluxes. This understanding suggests we place a greater emphasis on higher noise levels, and we empirically validate this intuition in \cref{tab:r_sample,tab:r_range}.

\paragraph{Handling guidance in \cref{eq:c_grad}.}
The usual treatment for including guidance in \cref{eq:c_grad} is the same as in \cref{eq:p_grad}: replacing it with the guided velocity $\vu_\gamma$. However, we highlight that there is a subtle but major difference between $\vvG$ and $\vvN$. In a traditional flow model training with \cref{eq:flow_loss}, we essentially train the model to learn a dataset's noising velocity, and use it as the generating velocity during sampling, \ie, the two velocities are identical as they describe the same process. However, this equivalence no longer holds in the presence of techniques like CFG~\citep{karras2024guiding,zheng2024characteristic}, and the distinction is especially prominent at high noise levels. %
We resort to dropping the guidance application at high noise levels, in a similar fashion to the guidance interval~\citep{kynkaanniemi2024applying} used for flow sampling. Furthermore, as demonstrated in \cref{tab:guid_intvl}, we stress that their empirical behaviors are different, and one typically needs to limit the interval significantly more aggressively (the pre-trained SiT-B/2 does not benefit from guidance interval with $\gamma=2$) in our correction objective.

\paragraph{Adaptive gradient balancing between \cref{eq:p_grad,eq:c_grad}.}
We now discuss how to fuse the training signals from \cref{eq:p_grad,eq:c_grad}. First, in a similar manner to prior works~\citep{frans2024one,geng2025mean}, we decide to split the training batch between the prediction and correction objectives ($75\%$ and $25\%$, respectively), since correction is slightly more expensive compute-wise. Then, we adopt an adaptive gradient balancing strategy~\citep{esser2021taming}, where the correction gradients are scaled by some dynamic weight $\lambda$ before concatenating with the prediction gradients. With the form similarity between \cref{eq:p_grad} and \cref{eq:c_grad}, and the observation that both optimizations lack aleatoric uncertainty~\citep{kendall2017uncertainties}\footnote{$\Delta_{\vvG,\vu}$ and $\Delta_{\vvN,\vu}$ represent the quality of alignments, unlike \cref{eq:flow_loss}.}, we design $\lambda=\alpha\frac{\E\|\Delta_{\vvG,\vu}\|}{\E\|\Delta_{\vvN,\vu}\|+\epsilon}$, where the expectation is taken over the mini-batch and $\epsilon=10^{-6}$ is used for numerical stability. We show that the model performance is robust across a wide range of $\alpha$ in \cref{fig:alpha}.%

\paragraph{Gradient norm manipulations in \cref{eq:p_grad}.}
We note that we can change the magnitude of $\Delta_{\vvG,\vu}$ freely without changing the optimal solution, which can also be understood as changing the loss metrics~\cite{geng2024consistency}. Concretely, we can scale $\Delta_{\vvG,\vu}$ with per-sample positive weights. Prior works~\citep{song2024improved,geng2024consistency,geng2025mean} mostly explored scaling with $1/(\|\Delta_{\vvG,\vu}\|^2+\varepsilon)^k$, where $\varepsilon=10^{-4}$ is used for numerical stability and we vary the power term $k$. Since the actual weighting applied depends on the original norm of $\Delta_{\vvG,\vu}$, which changes with the dimension of data $d$, we first divide it by $\sqrt{d}$ so that it is dimension invariant before calculating the weight. We note that such a weighting design corresponds to applying a power-law decay on $\Delta_{\vvG,\vu}$, and a larger $k$ indicates a stronger decay, effectively dampening the gradient contributions. In \cref{tab:k_impact}, we find that the model prefers weightings with a stronger decay when training with both \cref{eq:p_grad,eq:c_grad}. We hypothesize that this is because their signals may not always agree with each other in practice, and by applying a dampener on $\Delta_{\vvG,\vu}$, we mitigate the conflict between the two objectives and promote a more harmonious joint optimization. %

\subsection{Main Results}
\label{sec:result}

\begin{table*}[t]
    \centering
    \caption{
        \textbf{Class-conditional generation on ImageNet 256$\bm{\times}$256 and 512$\bm{\times}$512.} \textsuperscript{*} indicates the use of AutoGuidance~\citep{karras2024guiding}. Methods marked with \textsuperscript{$\dagger$} are initialized from pretrained models. Crucially, unlike other listed distillation baselines, ours is constructed to be entirely data-free.
    }
    \label{tab:main}
    \scriptsize
    \setlength{\tabcolsep}{8pt}

    \begin{subtable}[t]{0.48\linewidth}
    \centering
    \begin{NiceTabular}[t]{
        l
        c
        c
        c
        S[table-format=2.2, detect-weight=true, mode=text]
    }
        \multicolumn{5}{l}{\textbf{Class-Conditional ImageNet 256$\bm{\times}$256}} \\
        \toprule
        \textbf{Method} &  \textbf{Epochs} & \textbf{\#Params} & \textbf{NFE} $\downarrow$ & \multicolumn{1}{c}{\textbf{FID} $\downarrow$} \\
        \midrule
        \multicolumn{5}{l}{\textit{\textbf{Teacher Diffusion / Flow Models}}} \\
        \arrayrulecolor{black!30}\midrule
        SiT-XL/2~\citep{ma2024sit} & 1400 & 675M & 250${\times}$2 & 2.06 \\
        SiT-XL/2+REPA~\citep{yu2024representation} & 800 & 675M & 434 & 1.37 \\
        \arrayrulecolor{black}\midrule

        \multicolumn{5}{l}{\textit{\textbf{Fast Flow from scratch}}} \\
        \arrayrulecolor{black!30}\midrule
        \multirow{2}{*}{Shortcut-XL/2~\citep{frans2024one}} & \multirow{2}{*}{250} &  \multirow{2}{*}{675M} & 1 & 10.60 \\
         &  &  & 128 & 3.80 \\
        \arrayrulecolor{black!10}\midrule
        \multirow{2}{*}{IMM-XL/2~\citep{zhou2025inductive}} & \multirow{2}{*}{3840}  & \multirow{2}{*}{675M} & 1${\times}$2 & 7.77 \\
         &  &  & 8${\times}$2 & 1.99 \\
        \arrayrulecolor{black!10}\midrule
        \multirow{2}{*}{STEI~\citep{liu2025learning}} & \multirow{2}{*}{1420\textsuperscript{$\dagger$}} & \multirow{2}{*}{675M} & 1 & 7.12 \\
         &  &  & 8 & 1.96 \\
        \arrayrulecolor{black!10}\midrule
        \multirow{2}{*}{MeanFlow-XL/2~\citep{geng2025mean}}  &  240 & \multirow{2}{*}{676M} & 1 & 3.43 \\
         & 1000 &  & 2 & 2.20 \\
        \arrayrulecolor{black!10}\midrule
        \multirow{2}{*}{DMF-XL/2~\citep{lee2025decoupled}}  & \multirow{2}{*}{880\textsuperscript{$\dagger$}} & \multirow{2}{*}{675M} & 1 & 2.16 \\
         &  &  & 4 & 1.51 \\
        \arrayrulecolor{black}\midrule

        \multicolumn{5}{l}{\textit{\textbf{Fast Flow by distillation}}} \\
        \arrayrulecolor{black!30}\midrule
        \multicolumn{5}{l}{\cellcolor{gray!10}\textit{Teacher: SiT-XL/2 (FID $=2.06$)}} \\
        \arrayrulecolor{black!10}\midrule
        SDEI~\citep{liu2025learning} & 20 &  675M & 8 & 2.46 \\
        \arrayrulecolor{black!10}\midrule
        FACM~\citep{peng2025flow} & --  & 675M & 2 & 2.07 \\
        \arrayrulecolor{black!10}\midrule
        \multirow{2}{*}{\textbf{FreeFlow-XL/2}} & 20  &  \multirow{2}{*}{678M} & 1 &  2.24 \\
         & 300 &  & 1 & \bfseries  1.69 \\
        \arrayrulecolor{black!30}\midrule
        \multicolumn{5}{l}{\cellcolor{gray!10}\textit{Teacher: SiT-XL/2+REPA (FID $=1.37$)}} \\
        \arrayrulecolor{black!10}\midrule
        FACM~\citep{peng2025flow} & -- & 675M & 2 & 1.52 \\
        \arrayrulecolor{black!10}\midrule
        \multirow{2}{*}{$\pi$-Flow~\citep{chen2025pi}} & \multirow{2}{*}{448} & \multirow{2}{*}{675M} & 1 & 2.85 \\
         &  &  & 2 & 1.97 \\
        \arrayrulecolor{black!10}\midrule
        \multirow{2}{*}{\textbf{FreeFlow-XL/2}} & 20  &  \multirow{2}{*}{678M} & 1 &  1.84 \\
         &  300  &  & 1 & \bfseries  1.45 \\
        \arrayrulecolor{black}\bottomrule
    \end{NiceTabular}
    \end{subtable}
    \hfill
    \begin{subtable}[t]{0.48\linewidth}
    \centering
        \begin{NiceTabular}[t]{
        l
        c
        c
        c
        S[table-format=2.2, detect-weight=true, mode=text]
    }
        \multicolumn{5}{l}{\textbf{Class-Conditional ImageNet 512$\bm{\times}$512}} \\
        \toprule
        \textbf{Method} & \textbf{Epochs} & \textbf{\#Params} & \textbf{NFE} $\downarrow$ & \multicolumn{1}{c}{\textbf{FID} $\downarrow$} \\
        \midrule
        \multicolumn{5}{l}{\textit{\textbf{Teacher Diffusion / Flow Models}}} \\
        \arrayrulecolor{black!30}\midrule
        SiT-XL/2~\citep{ma2024sit} & 600 & 675M & 250${\times}$2 & 2.62 \\
        SiT-XL/2+REPA~\citep{yu2024representation} &  400 &  675M & 460 & 1.37  \\
        \arrayrulecolor{black!10}\midrule
        EDM2-S\textsuperscript{*}~\citep{karras2024guiding} & 1678 & 280M & 63${\times}$2 & 1.34 \\
        \arrayrulecolor{black!10}\midrule
        EDM2-XXL~\citep{karras2024analyzing} &  \multirow{2}{*}{734}  & \multirow{2}{*}{1.5B} & 82 & 1.40 \\
        EDM2-XXL\textsuperscript{*}~\citep{karras2024guiding}   &   &  & 63${\times}$2 & 1.25 \\
        \arrayrulecolor{black}\midrule

        \multicolumn{5}{l}{\textit{\textbf{Fast Flow from scratch}}} \\
        \arrayrulecolor{black!30}\midrule
        \multirow{2}{*}{sCT-XXL~\citep{lu2024simplifying}} & \multirow{2}{*}{761\textsuperscript{$\dagger$}}  & \multirow{2}{*}{1.5B} & 1 & 4.29 \\
         & & & 2 & 3.76 \\
        \arrayrulecolor{black!10}\midrule
        \multirow{2}{*}{DMF-XL/2~\citep{lee2025decoupled}} & \multirow{2}{*}{540\textsuperscript{$\dagger$}}  & \multirow{2}{*}{675M} & 1 & 2.12 \\
         &  &  & 4 & 1.68 \\
        \arrayrulecolor{black}\midrule

        \multicolumn{5}{l}{\textit{\textbf{Fast Flow by distillation}}} \\
        \arrayrulecolor{black!30}\midrule

        \multicolumn{5}{l}{\cellcolor{gray!10}\textit{Teacher: EDM2-S\textsuperscript{*} (FID $=1.34$)}} \\
        \arrayrulecolor{black!10}\midrule
        \multirow{2}{*}{AYF-S~\citep{sabour2025align}} &  \multirow{2}{*}{80}  &  \multirow{2}{*}{280M} & 1 & 3.32 \\
         & & & 4 & 1.70 \\

        \arrayrulecolor{black!30}\midrule
        \multicolumn{5}{l}{\cellcolor{gray!10}\textit{Teacher: EDM2-XXL (FID $=1.40$)}} \\
        \arrayrulecolor{black!10}\midrule
        \multirow{2}{*}{sCD-XXL~\citep{lu2024simplifying}} &   \multirow{2}{*}{320} &  \multirow{5}{*}{1.5B} & 1 & 2.28 \\
         & & & 2 & 1.88 \\
        \arrayrulecolor{black!10}\cmidrule{1-2}\cmidrule{4-5}
        \multirow{2}{*}{sCD-XXL+VSD~\citep{lu2024simplifying}} & \multirow{2}{*}{32} &  & 1 & 2.16 \\
         & & & 2 & 1.89 \\
        \arrayrulecolor{black!30}\midrule

        \multicolumn{5}{l}{\cellcolor{gray!10}\textit{Teacher: SiT-XL/2 (FID $=2.62$)}} \\
        \arrayrulecolor{black!10}\midrule
        \multirow{2}{*}{\textbf{FreeFlow-XL/2}} & 20 & \multirow{2}{*}{678M} & \multirow{2}{*}{1} & 3.01 \\
         & 200 & & & \bfseries 2.25 \\
        \arrayrulecolor{black!30}\midrule
        \multicolumn{5}{l}{\cellcolor{gray!10}\textit{Teacher: SiT-XL/2+REPA (FID $=1.37$)}} \\
        \arrayrulecolor{black!10}\midrule
        \multirow{2}{*}{\textbf{FreeFlow-XL/2}} &  20  & \multirow{2}{*}{678M} & \multirow{2}{*}{1} & 2.11 \\
         & 200 & & & \bfseries 1.49 \\
        \arrayrulecolor{black}\bottomrule
    \end{NiceTabular}
    \end{subtable}
    \vspace{-0.3cm}
\end{table*}

\paragraph{Comparisons with prior work.}
In \cref{tab:main}, we benchmark our approach against existing proposals for learning fast flows on class-conditional ImageNet~\citep{russakovsky2015imagenet} generation at both 256${\times}$256 and 512${\times}$512 resolutions. We highlight three key findings from our main results. (1) \textbf{Our method achieves state-of-the-art performance by a significant margin.} Distilling from SiT-XL/2+REPA~\citep{yu2024representation}, our method reaches an impressive FID of \textbf{1.45} on 256${\times}$256 and \textbf{1.49} on 512${\times}$512, greatly outperforming prior proposals. (2) \textbf{Competitive performance is realized very early in training.} Our model surpasses the final performance of many strong baselines after only 100K iterations ($\approx$20 epochs), demonstrating exceptional training efficiency. (3) \textbf{Our student faithfully reproduces the teacher's full capabilities with 1-NFE.} Our distilled model consistently stays within 10\% of the teacher's original performance with only a single step, even when the teacher employs advanced training techniques like REPA~\citep{yu2024representation}, and sampling techniques like guidance intervals~\citep{kynkaanniemi2024applying}. This is a critical advantage over training fast flows from scratch: it allows us to seamlessly inherit the benefits of complex teacher training recipes, which can be complicated to adopt~\citep{lee2025decoupled}, simply by distilling the final product. Crucially, we emphasize that our results are achieved using \emph{only} the pre-trained teacher model, without requiring access to a single sample from an external dataset, real or synthetic, thus validating the effectiveness of the data-free paradigm.

\paragraph{Synergy between prediction and correction.}
While theory suggests that either learning the flow trajectories (\cref{eq:p_grad}) or matching the marginal distributions (\cref{eq:c_grad}) could suffice for generation, we find that neither is robust in isolation. The prediction objective, when used alone, falls victim to the error accumulation identified in \cref{sec:predict}, plateauing at suboptimal fidelity ({\color[HTML]{5D8DFD} blue line} in \cref{fig:synergy}). Conversely, training only with the correction objective (\cref{eq:c_grad}) leads to gradual mode collapse and performance degradation ({\color[HTML]{88B06D} green line} in \cref{fig:synergy}; also {\color[HTML]{ADB5BD} gray baseline} in \cref{fig:alpha}). \cref{fig:synergy} illustrates the powerful synergy realized by our framework on a SiT-XL/2 teacher. By combining both signals at their optimal settings, we achieve performance strictly superior to either independent component. The prediction signals construct the generative path, while the correction signals act as a stabilizer to rectify compounding errors, ensuring consistent improvement throughout training.

\paragraph{Inference-time scaling.}
The recently proposed inference-time scaling framework~\citep{ma2025inference, singhal2025general} offers a promising avenue to trade additional compute for generation quality. However, existing search strategies typically require the full integration of $\vphi_\vu$ for every candidate, making the search process prohibitively expensive. We propose a more efficient alternative: by distilling the teacher into a flow map, we create a fast proxy that retains the teacher's mapping from noise to data. This allows us to conduct the expensive search using the cheap, one-step student, transferring only the optimal noise to the teacher for final generation. We investigate a Best-of-$N$ search with an oracle verifier~\citep{ma2025inference}, employing our student (trained for only 20 epochs) to guide the fixed SiT-XL/2 teacher. As shown in \cref{fig:tts}, this approach drastically improves the teacher's sampling quality. Crucially, the results highlight the benefit of our prediction objective (\cref{eq:p_grad}): while the correction-only model (\cref{eq:c_grad}) yields improvements, it is notably less efficient at identifying transferable noise candidates due to a lack of guaranteed trajectory alignment. Our combined objective, by enforcing strict consistency with $\vphi_\vu$, enables a much more effective search. With a total budget of only 80 NFEs, our method outperforms the teacher's standard classifier-free guidance sampling at 128 NFEs. This result demonstrates a powerful practical trade-off: by shifting a fraction of the inference burden to a short distillation phase, we enable the compute-efficient deployment of large-scale diffusion models.

\section{Conclusion}
\label{sec:conclusion}

In this work, we challenge the conventional reliance on external datasets for flow map distillation. We identify a fundamental vulnerability in this practice, the Teacher-Data Mismatch, and argue that a static dataset is an inherently unreliable proxy for a teacher's full generative capabilities. This data-dependency is not only risky but also unnecessary. Our principled, data-free alternative, formulated as a predictor-corrector framework, resolves this mismatch by construction as it samples only from the prior. Our strong empirical results, which establish a new state-of-the-art, confirm the practical viability and strength of this data-free paradigm. We believe this work provides a more robust foundation for accelerating generative models and hope it motivates a broader exploration of flow map distillation without data.

\section*{Acknowledgments}
We are grateful to Kaiming He for valuable discussions and feedback on the manuscript. This work was partly supported by the Google TPU Research Cloud (TRC) program and the Google Cloud Research Credits program (GCP19980904). ST and TJ acknowledge support from the Machine Learning for Pharmaceutical Discovery and Synthesis (MLPDS) consortium, the DTRA Discovery of Medical Countermeasures Against New and Emerging (DOMANE) threats program, the NSF Expeditions grant (award 1918839) Understanding the World Through Code. SX acknowledges support from the MSIT IITP grant (RS-2024-00457882) and the NSF award IIS-2443404.

{
    \small
    \bibliographystyle{ieeenat_fullname}
    \bibliography{main}
}

\clearpage
\onecolumn

\appendix
\crefalias{section}{appendix}
\crefalias{subsection}{appendix}

\input{app}

\end{document}

%% file: preamble.tex
\renewcommand{\paragraph}[1]{\vspace{.5em}\noindent\textbf{#1}}

\usepackage{soul}
\setuldepth{foobar}

\usepackage[utf8]{inputenc}                 %
\usepackage[T1]{fontenc}                    %
\usepackage[english]{babel}                 %
\usepackage{textcomp}                       %
\usepackage{microtype}                      %
\usepackage{amsthm}                         %
\usepackage{listings}                       %
\usepackage{xurl}                           %

\usepackage{csquotes}                       %
\usepackage{mathtools}                      %
\usepackage{bm}                             %
\usepackage{nicefrac}                       %
\usepackage{siunitx}                        %
\usepackage{multirow}                       %
\usepackage{colortbl}                       %
\usepackage{wrapfig}                        %
\usepackage{algorithm,algpseudocode}        %
\usepackage{nicematrix}                     %
\usepackage{pifont}                         %
\usepackage{fontawesome5}                   %

\input{glyphtounicode}
\input{cmd}

\newcommand\numberthis{\addtocounter{equation}{1}\tag{\theequation}}

\def\vphi{{\bm{\phi}}}
\def\vF{{\bm{F}}}
\def\vG{{\bm{G}}}
\def\vI{{\bm{I}}}
\def\sg{{\operatorname{sg}}}

\def\vvG{{\vv_\text{G}}}
\def\vvN{{\vv_\text{N}}}

\theoremstyle{plain}
\newtheorem{theorem}{Theorem}[section]
\newtheorem{proposition}[theorem]{Proposition}

\theoremstyle{definition}

\theoremstyle{remark}

\newcommand{\tablestyle}[2]{\setlength{\tabcolsep}{#1}\renewcommand{\arraystretch}{#2}\centering\scriptsize}

%% file: cmd.tex
\usepackage{amsmath,amsfonts,bm}

\def\eqref#1{equation~\ref{#1}}

\def\1{\bm{1}}

\def\vf{{\bm{f}}}
\def\vg{{\bm{g}}}

\def\vn{{\bm{n}}}

\def\vu{{\bm{u}}}
\def\vv{{\bm{v}}}

\def\vx{{\bm{x}}}
\def\vy{{\bm{y}}}
\def\vz{{\bm{z}}}

\def\mI{{\bm{I}}}

\DeclareMathAlphabet{\mathsfit}{\encodingdefault}{\sfdefault}{m}{sl}
\SetMathAlphabet{\mathsfit}{bold}{\encodingdefault}{\sfdefault}{bx}{n}

\def\gN{{\mathcal{N}}}
\def\gO{{\mathcal{O}}}

\def\gU{{\mathcal{U}}}

\newcommand{\E}{\mathbb{E}}

\newcommand{\R}{\mathbb{R}}

%% file: app.tex
\section{Extended Technical Discussion}
\label{sec:ext_discuss}

We continue our technical discussion in \cref{sec:method,sec:exp}, dividing this section into three parts.

\subsection{More on the Prediction Objective}
\label{sec:ext_pred}

\paragraph{Optimality of training with \cref{eq:p_loss_c}.}
Recall that we have already established that the loss value of \cref{eq:p_loss_c} equals $\E_{\vz,\delta}\| \partial_\delta \vf_{\theta}(\vz,\delta) - \vu(\vf_{\theta}(\vz,\delta),1-\delta) \|^2$. Thus, at optimality, we have $\partial_\delta \vf_{\theta}(\vz,\delta) = \vu(\vf_{\theta}(\vz,\delta),1-\delta)$ for all $\vz\sim\pi$ and $\delta\in[0,1]$. Furthermore, we note that the parametric relationship $\vf_\theta(\vz,\delta)=\vz+\delta\vF_\theta(\vz,\delta)$, which satisfies $\vf_\theta(\vz,0)=\vz$ for all $\vz\sim\pi$ by design. We show that, assuming standard regularity conditions on $\vu$, we ensure $\vphi_\vu$ is uniquely defined via an initial value problem by the Picard-Lindelöf theorem~\citep{coddington1956theory}, thus providing a well-defined criterion for $\vf_\theta$. This is a direct result of the fundamental theory of ODE.

\begin{proposition}
\label{thm:condition}
    Let $\vphi_\vu(\vx_t,t,s)$ be defined as in \cref{eq:solution}, and we assume that there exists some $L>0$ for all $\vy, \vy'\in\R^d$ and $r\in[0,1]$, $\|\vu(\vy,r)-\vu(\vy',r)\|\leq L\|\vy-\vy'\|$. We further assume that $\vf_\theta(\vz,\delta)$ is continuously differentiable for $\delta\in[0,1]$. Then, we have $\vf_\theta(\vz,\delta) = \vphi_\vu(\vz,1,1-\delta)$ for all $\vz\sim\pi,\delta\in[0,1]$, if and only if, (i) $\vf_\theta(\vz,0)=\vz$ for all $\vz\sim\pi$, and (ii) $\partial_\delta \vf_\theta(\vz, \delta) = \vu(\vf_\theta(\vz, \delta), 1-\delta)$ for all $\vz\sim\pi,\delta\in[0,1]$.
\end{proposition}
\begin{proof}
By assuming $\vu(\vy, r)$ continuous in $r$ and Lipschitz continuous in $\vy$, we could apply the Picard-Lindelöf theorem~\citep{coddington1956theory} to guarantee the existence and uniqueness of its solution $\vphi_\vu$, which is %
\begin{align*}
    \vphi_\vu(\vx_t,t,s) = \vx_t + \int_t^{s} -\vu(\vx(\tau), \tau)d\tau.
\end{align*}

$\Longrightarrow$:

(i). Since the equality $\vf_\theta(\vz,\delta) = \vphi_\vu(\vz,1,1-\delta)$ holds for all $\delta\in[0,1]$, it must hold for $\delta=0$. We know that $\vphi_\vu(\vz,1,1)=\vz$. Thus, for any $\vz\sim\pi$,
\begin{align*}
    \vf_\theta(\vz,0) = \vphi_\vu(\vz,1,1)=\vz.
\end{align*}

(ii). Given that $\vf_\theta(\vz,\delta) = \vphi_\vu(\vz,1,1-\delta)$ for all $\delta\in[0,1]$, and both $\vf_\theta(\vz,\delta)$ and $\vphi_\vu(\vz,1,1-\delta)$ are continuously differentiable with respect to $\delta$ on $[0,1]$, their derivatives with respect to $\delta$ must be equal. That is, for all $\vz\sim\pi,\delta\in[0,1]$,
\begin{align*}
    \partial_\delta\vf_\theta(\vz,\delta) = \frac{d}{d\delta}\vphi_\vu(\vz,1,1-\delta) = \vu(\vphi_\vu(\vz,1,1-\delta),1-\delta) = \vu(\vf_\theta(\vz,\delta),1-\delta).
\end{align*}

$\Longleftarrow$:

Since $\vf_\theta(\vz,0)=\vz$ for all $\vz\sim\pi$, and $\partial_\delta \vf_\theta(\vz, \delta) = \vu(\vf_\theta(\vz, \delta), 1-\delta)$ for all $\vz\sim\pi,\delta\in [0,1]$, we know that $\vf_\theta(\vz,\delta)$ is a solution to the initial value problem on the interval $\delta\in[0,1]$:
\begin{align*}
    \frac{d}{d\delta}\vf_\theta(\vz,\delta) = \vu(\vf_\theta(\vz,\delta),1-\delta),\quad \vf_\theta(\vz,0) = \vz.
\end{align*}
By definition, $\vphi_\vu(\vz,1,1-\delta)$ is also a solution to the same IVP on $[0,1]$. Thus, we arrive at the equality between $\vf_\theta(\vz,\delta)$ and $\vphi_\vu(\vz,1,1-\delta)$.

\end{proof}

\paragraph{Discrete-time objective.}
In \cref{sec:predict}, we skipped over the development of a more flexible discrete-time training objective for \cref{eq:p_loss_c}, which does not require efficient JVP implementations. First, recall that our default continuous-time objective is
\begin{align*}
    \E_{\vz,\delta}\left\|\, \vF_\theta(\vz,\delta) + \sg\Big( \delta\partial_\delta \vF_{\theta}(\vz,\delta) - \vu\left(\vf_{\theta}(\vz,\delta),1-\delta\right)\Big) \,\right\|^2.\numberthis\label{eq:p_loss_c_app}
\end{align*}
Specifically, we can discretize the time horizon and approximate the partial derivative $\partial_\delta\vF_\theta$ using finite differences. Following the common practice in the sampling procedure of flow models~\citep{ho2020denoising}, we divide the time horizon into $N$ intervals with $N+1$ boundary points: $1=t_1>t_2>\dots>t_{N+1}=0$. We note that the generating velocity at point $\vf_\theta(\vz,1-t_i)$, or equivalently $\vf_\theta(\vz,t_1-t_i)$ where $1\leq i\leq N$, can be approximated by $\frac{1}{t_i-t_{i+1}}(\vf_\theta(\vz,t_1-t_{i+1})-\vf_\theta(\vz,t_1-t_i))$. Short-handing $t_a-t_b$ to $\delta_{a,b}$, we have the following discrete-time objective:
\begin{align*}
    \E_{\vz,i} \left\|\, \vF_\theta(\vz,\delta_{1,i+1}) + \sg\left( \delta_{1,i}\cdot \frac{\vF_{\theta}(\vz,\delta_{1,i+1}) - \vF_{\theta}(\vz,\delta_{1,i})}{\delta_{i,i+1}} - \vu\left(\vf_{\theta}(\vz,\delta_{1,i}), t_i\right) \right) \,\right\|^2,\numberthis\label{eq:p_loss_d}
\end{align*}
which approximates \cref{eq:p_loss_c} as $\delta_{i,i+1}\to0$. Just like \cref{eq:p_loss_c_app} equals to $\E_{\vz,\delta}\| \partial_\delta \vf_{\theta}(\vz,\delta) - \vu(\vf_{\theta}(\vz,\delta),1-\delta) \|^2$, we note that the loss value of \cref{eq:p_loss_d} is the same as $\E_{\vz,i}\left\| \left(\vf_{\theta}(\vz,\delta_{1,i+1}) - \vf_{\theta}(\vz,\delta_{1,i})\right) / \delta_{i,i+1} - \vu\left(\vf_{\theta}(\vz,\delta_{1,i}), t_i\right) \right\|^2$, where $\left(\vf_{\theta}(\vz,\delta_{1,i+1}) - \vf_{\theta}(\vz,\delta_{1,i})\right) / \delta_{i,i+1}$ is the model's generating velocity, approximating $\partial_\delta\vf_\theta(\vz,\delta)$. If we use $\vvG$ to further simplify the notations, we arrive at the following optimization gradients in \cref{eq:p_grad}:
\begin{align*}
    \nabla_\theta\, \E_{\vz,\delta}\biggl[ \vF_\theta(\vz,\delta)^\top \sg\Bigl( \underbrace{\vvG\bigl(\vf_{\theta}(\vz,\delta),1-\delta\bigr) - \vu\bigl(\vf_{\theta}(\vz,\delta),1-\delta\bigr)}_{\Delta_{\vvG,\vu}(\vf_{\theta}(\vz,\delta),1-\delta)} \Bigr) \biggr],\numberthis\label{eq:p_grad_app}
\end{align*}

\paragraph{Error analysis.}
With a discrete-time objective, we essentially train the student to trace a numerically integrated trajectory by an ODE solver. More concretely, the loss function inherently mimics a specific solver, whose error rate depends on the technique deployed. As written, \cref{eq:p_loss_d} can be understood as using an Euler method~\citep{song2020denoising,song2021scorebased} for solving the flow, specified at time steps $t_1, t_2, \dots, t_N$. Similar to existing fast-solver literature, we could readily adapt the objective using higher-order finite differences to mimic more precise, higher-order solvers~\citep{lu2022dpm,karras2022elucidating}. The classical theory from numerical analysis~\citep{suli2003introduction} indicates that a local truncation error at $t_i$ bounded by $\gO((\delta_{i,i+1})^{p+1})$ leads to the global error rate is $\gO((\delta_{\max})^p)$, where $\delta_{\max} \coloneqq \max_i\delta_{i,i+1}=\max_i(t_i - t_{i+1})$ is the maximum step size. This component is the discretization error, which is entirely independent of the network's approximation error (the inevitable training inaccuracy discussed in \cref{sec:predict}). Together, these two errors represent the total deviation of the student's trajectory from the true teacher flow defined by $\vu$. The total error can also be shown~\citep{boffi2025flow} to directly control the Wasserstein distance between the generative distribution of the student and the teacher.

\begin{wrapfigure}{r}{0.33\textwidth}
    \centering
    \includegraphics[width=\linewidth]{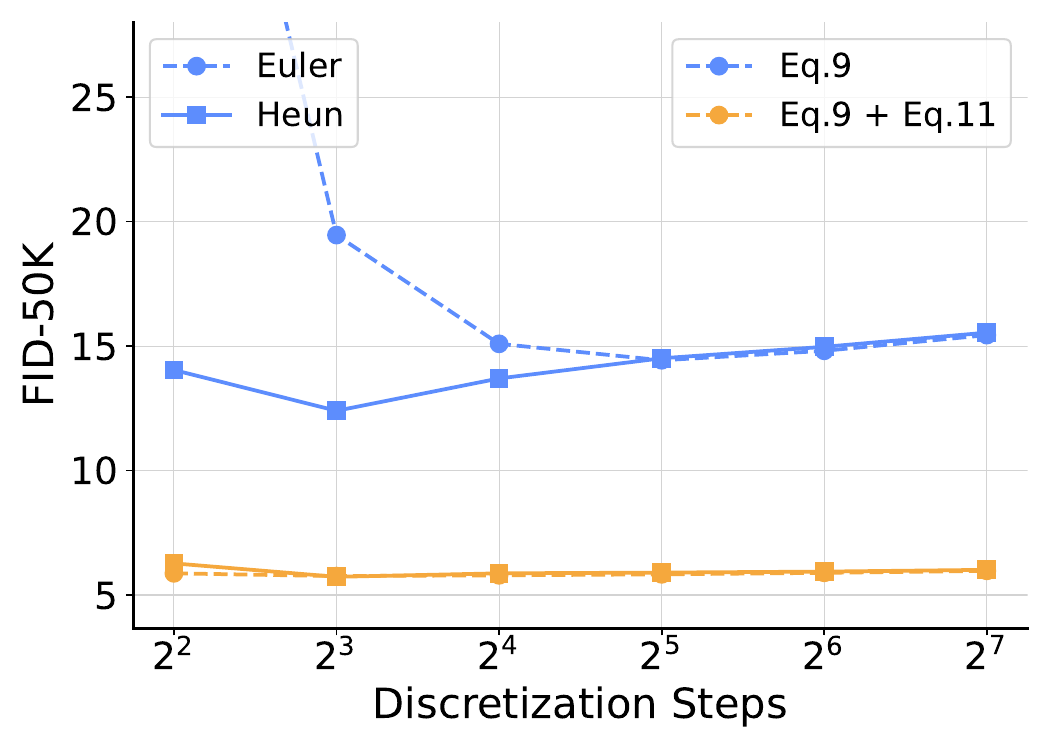}
    \caption{\textbf{Effects of discretization steps and solver type.} Our correction objective in \cref{eq:c_grad} makes the model robust against both.}
    \label{fig:discrete}
\end{wrapfigure}
\paragraph{Number of discretization steps.}
The above discussion suggests that one should take as small a step as possible, $(t_i - t_{i+1})\to0$, and as precise a solver as possible, for the discretization error to be minimal. This means that we should use the continuous-time setup if allowed, and for the discrete-time setting, we should choose to use a large number of discretization steps $N$. However, this is not the trend we observe in practice. Surprisingly, larger $N$ after a certain point actually degrades the performance. We attribute this behavior to the accumulated approximation error (discussed in \cref{sec:predict}) made by the model. Since $\vu$ is evaluated at the model's prediction, the goodness of the trajectory implicitly depends on the quality of $\vf_\theta$ itself, and such a problem exacerbates as $N$ becomes larger. We further note that incorporating higher-order solvers in $\vu$ like the Heun method~\citep{karras2022elucidating} is helpful, as it significantly reduces the discretization error. Fortunately, in \cref{fig:discrete}, we observe that while the model is sensitive to the number of discretization steps and solver choices when training with \cref{eq:p_grad} alone, the corrective signal in \cref{eq:c_grad} effectively renders such a decision unimportant, making our final proposal robust against $N$ and the precision of solvers.

\begin{wraptable}{r}{0.33\textwidth}
    \centering
    \tablestyle{4pt}{1.1}
    \footnotesize
    \begin{tabular}[t]{S[table-format=2.0, table-space-text-post=\%, mode=text]<{\%}lS[table-format=2.2, mode=text]}
        \toprule
        \multicolumn{1}{c}{\% of $\delta=0$} & $\delta$ sampling & \multicolumn{1}{c}{FID $\downarrow$} \\
        \midrule
        10 & LogitNormal(0, 1) & 12.40 \\
        \arrayrulecolor{black!30}\specialrule{\lightrulewidth}{1pt}{1pt}\arrayrulecolor{black}
        0 & LogitNormal(0, 1) & 47.09 \\
        30 & LogitNormal(0, 1) & 13.19 \\
        50 & LogitNormal(0, 1) & 14.30 \\
        \arrayrulecolor{black!30}\specialrule{\lightrulewidth}{1pt}{1pt}\arrayrulecolor{black}
        10 & LogitNormal(-0.8, 1) & 13.78 \\
        10 & LogitNormal(-0.4, 1) & 12.98 \\
        10 & LogitNormal(0, 1.2) & 12.60 \\
        \arrayrulecolor{black!30}\specialrule{\lightrulewidth}{1pt}{1pt}\arrayrulecolor{black}
        0 & LogitNormal(-0.8, 1.6) & 13.43 \\
        0 & LogitNormal(-0.4, 1.6) & 12.98 \\
        0 & LogitNormal(0, 1.6) & 12.59 \\
        \bottomrule
    \end{tabular}
    \caption{\textbf{Ablation of $\delta$ sampling.} A mixture of a logit-normal and a fixed $\delta{=}0$ works well.}
    \label{tab:delta_dist}
\end{wraptable}
\paragraph{Sampling of $\delta$.}
Another important aspect of the training algorithm is the time sampling of $\delta$. On one hand, the correctness of the model propagates from small jumps $\delta=0$ to large jumps $\delta=1$. On the other hand, the high-frequency features of the images only emerge at a lower noise level (large $\delta$). To balance these two notions, we investigate a mixture of a logit-normal distribution~\citep{karras2022elucidating,esser2024scaling} and a fixed value of $\delta=0$ (effectively a dropout on $\delta$). Results are presented in \cref{tab:delta_dist}. Note that for the discrete-time setting, we apply an additional step of the floor operation with respect to our predefined set of discrete time steps.

\paragraph{Confident region warmup.}
We observe that naively optimizing with \cref{eq:p_grad} brings about instability early on in training, mainly with the JVP approach. We hypothesize that it is because $\vu$ is evaluated at a state predicted by $\vf_\theta$, which can be out-of-distribution for $\vu$ at initialization. We propose to add noise to the predicted state before feeding it to $\vu$. Specifically, we replace $\vu(\vf_{\theta}(\vz,\delta),1{-}\delta)$, with $\vu(\vI_{t_c|1{-}\delta}(\vf_{\theta}(\vz,\delta),\vn), t_c)$, where $\vI_{t_c|t}$ is the generalized interpolating function, a transition kernel that takes samples from the noise level of $t$ to $t_c$. For the linear interpolation scheme~\citep{lipman2023flow,liu2023flow,ma2024sit}, we have $\vI_{t_c|t}(\vx_t,\vn)=\frac{1-t_c}{1-t}\vx_t+\sqrt{t_c^2 - \frac{(1-t_c)^2}{(1-t)^2}t^2}\vn$, assuming $t_c>t$; if $t_c\leq t$, it simply does nothing and returns $\vx_t$. At the start of training, we linearly decrease $t_c$ from 1 to 0 over a short warmup of 10K steps. Intuitively, this procedure ensures that $\vu$ always operates in a region where it is confident. After the inclusion of this warmup period, we do not find any instability of training $\vf_\theta$ with a reasonable sampling distribution of $\delta$.

\subsection{More on the Correction Objective}
\label{sec:ext_corr}

\paragraph{From minimizing IKL to aligning the noising velocities.}
We know from prior works~\citep{luo2023diff,yin2024one} that minimizing \cref{eq:ikl} \wrt $\theta$ gives us the following gradient:
\begin{align*}
    \E_{r,\vx_r}\left[ \left(\nabla_\theta \vx_r \right)^\top \left(\nabla_{\vx_r} \log q_r(\vx_r) - \nabla_{\vx_r} \log p_r(\vx_r)  \right) \right],\numberthis\label{eq:ikl_grad}
\end{align*}
where $\nabla_{\vx_r}\log q_r(\vx_r)$ and $\nabla_{\vx_r}\log p_r(\vx_r)$ are the score functions~\citep{song2019generative,hyvarinen2005estimation} of $q_r$ and $p_r$ respectively. It has also been shown that there exists a direct bijective translation between the score functions and the marginal velocities~\citep{ma2024sit}. Specifically, for the linear interpolation scheme~\citep{lipman2023flow,liu2023flow,ma2024sit} and our definition of the conditional velocity, we have:
\begin{align*}
    \vu(\vx_r,r) &= \frac{r}{1-r}\nabla_{\vx_r}\log p_r(\vx_r) + \frac{1}{1-r}\vx_r \\
    \nabla_{\vx_r}\log p_r(\vx_r) &= \frac{1-r}{r} \vu(\vx_r,r) - \frac{1}{r}\vx_r.
\end{align*}
Additionally, recall that the student predicts the data sample as $\vf_\theta(\vz, 1)$. With average velocity parameterization, the intermediate state is thus $\vx_r = \vI_r(\vf_\theta(\vz,1),\vn) = (1-r)(\vz+\vF_\theta(\vz,1))+r\vn$. Thus, we could rewrite \cref{eq:ikl_grad} as
\begin{align*}
    \E_{r,\vx_r}\left[ \frac{(1-r)^2}{r} \left(\nabla_\theta \vF_\theta(\vz,1) \right)^\top \left( \vv_\text{N}\bigl(\vI_r(\vf_{\theta}(\vz,1), \vn), r\bigr) - \vu\bigl(\vI_r(\vf_{\theta}(\vz,1), \vn), r\bigr) \right) \right],
\end{align*}
where $\vvN$ is the marginal noising velocity induced by the student's generated distribution $q$, similar to $\vu$ with $p$. Dropping the weighting $\frac{(1-r)^2}{r}$, as it does not change the optimal solution and provides us with easier-to-control gradients, we arrive at \cref{eq:c_grad}. Note that a different interpolation scheme does not change our investigation.
\begin{align*}
    \nabla_\theta\, \E_{\vz,\vn,r} \biggl[ \vF_\theta(\vz,1)^\top \sg\Bigl( \underbrace{\vv_\text{N}\bigl(\vI_r(\vf_{\theta}(\vz,1), \vn), r\bigr) - \vu\bigl(\vI_r(\vf_{\theta}(\vz,1), \vn), r\bigr)}_{\Delta_{\vv_\text{N},\vu}(\vI_r(\vf_{\theta}(\vz,1), \vn), r)} \Bigr) \biggr].\numberthis\label{eq:c_grad_app}
\end{align*}

\begin{wraptable}{r}{0.33\textwidth}
    \centering
    \tablestyle{4pt}{1.1}
    \footnotesize
    \begin{tabular}[t]{cS[table-format=1.2, mode=text]} 
        \toprule
        \multicolumn{1}{c}{learning rate for $\vg_\psi$} & \multicolumn{1}{c}{FID $\downarrow$} \\
        \midrule
        \num{3e-5} & 8.28 \\
        \num{6e-5} & 5.77 \\
        \num{8e-5} & 5.72 \\
        \num{1e-4} & 5.63 \\
        \bottomrule
    \end{tabular}
    \vspace{0.5cm}
    \caption{\textbf{Ablation of $\vg_\psi$'s lr.} A higher lr compared to $\vf_\theta$'s one ($\num{3e-5}$) is better.}
    \label{tab:psi_lr}
\end{wraptable}
\paragraph{Learning rate.}
Our correction objective defined in \cref{eq:c_grad} assumes access to $\vvN$, which is the velocity of the noising flow starting from the generated distribution, and we approximate it by training an auxiliary model $\vg_\psi$ concurrently with \cref{eq:flow_loss}. The quality of the optimization signal $\Delta_{\vvN,\vu}$ depends on the quality of this approximation. Empirically in \cref{tab:psi_lr}, we confirm this intuition and observe that the algorithm benefits from having a larger learning rate on $\vg_\psi$ compared to $\vf_\theta$. We note that it is also possible to adopt a two time-scale update rule~\citep{heusel2017gans,yin2024improved}, where we train multiple iterations on $\vg_\psi$ before updating $\vf_\theta$, but we do not explore this option given the overhead. In a related vein, we find that while training with only the prediction objective in \cref{eq:p_grad} permits a wide range of learning rates, incorporating the correction objective in \cref{eq:c_grad} prefers a smaller one on $\vf_\theta$ (adopted $\num{3e-5}$ compared to $\num{1e-4}$ in SiT).

\paragraph{Additional note on sampling of $r$ in \cref{eq:c_grad}.}
Recall that our understanding of \cref{eq:c_grad} with a PDE perspective through the continuity equation in \cref{sec:design} leads to our design of sampling $r$ more in the higher noise levels. We would like to make a brief note that, similar to \cref{eq:p_grad}, if $\vF_\theta$ is further parameterized, we should replace the first term in \cref{eq:c_grad} with the actual network output. Additionally, one needs to ensure $\Delta_{\vvN,\vu}$ across different $r$ values are roughly on the same scale first, so that their actual contributions can be precisely controlled via the sampling of $r$. Concretely, for our case of the linear interpolation and velocity parameterization~\citep{lipman2023flow,liu2023flow,ma2024sit}, $\Delta_{\vvN,\vu}$ is really just the difference between the direct outputs of two neural networks, which does not require further manipulations, assuming the network outputs are of consistent scale. In comparison, another commonly used setup is the EDM parameterization~\citep{karras2022elucidating}, whose velocity at $(\vx_r=\vf_\theta(\vz,1) + r\vn, r), \vn\sim\pi$ is of the form $\frac{c_\text{skip}(r)-1}{r}\vx_r+\frac{c_\text{out}(r)}{r}\vG$, where $\vG$ is the actual network (the auxiliary or teacher model). Notice that the outputs of $\vG$ are by design of constant norm across $r$, so the magnitude of $\Delta_{\vvN,\vu}$ is proportional to $\frac{c_\text{out}(r)}{r}$. In \citet{yin2024one}, an additional weighting of $r$ is applied, which makes the overall gradient contributions from \cref{eq:c_grad} proportional to $c_\text{out}(r)$. Substituting in the actual terms used, we have $\frac{0.5r}{\sqrt{0.5^2+r^2}}$, which heavily downplays the effect of small $r$ (lower noise levels).

\subsection{More on How to Combine Both}
\label{sec:ext_both}

\paragraph{Additional schedule on $\alpha$.}
We set $\alpha=0$ for the first 10K steps (recall that the prediction objective has a warmup of 10K steps), and follow it with a linear warmup of another 10K steps before $\alpha$ settles at a reference value $\alpha_\text{ref}$. While the model performance is shown to be robust across a wide range of $\alpha$ in \cref{fig:alpha}, we further notice a general trend that larger $\alpha$ learns faster in the beginning of training, and smaller $\alpha$ converges better as the training continues. Hence, for our REPA distillation tasks, we try out a simple inverse square root decay~\citep{kingma2014adam,karras2024analyzing}, and arrive at the following schedule on $\alpha$:
\begin{align*}
    \alpha = \alpha_\text{ref}\frac{\operatorname{clip}\left(\frac{n-T_\text{delay}}{T_\text{warmup}}, 0, 1\right)}{\sqrt{\max(n/T_\text{decay}, 1)}},
\end{align*}
where $n$ is the current training iteration, $T_\text{delay}=T_\text{warmup}=10K$, and $T_\text{decay}$ is the hyperparameter that controls the decay rate with $\infty$ indicating no decay applied ($\alpha$ stays at the constant value $\alpha_\text{ref}$ after warmup). We also would like to clarify that, considering the adopted 75-25 split, an $\alpha$ value of $0.3$ really means that the gradient contribution of the correction objective is around $10\%$ of that of the prediction objective.

\begin{wraptable}{r}{0.33\textwidth}
    \centering
    \tablestyle{4pt}{1.1}
    \footnotesize
    \begin{tabular}[t]{lS[table-format=1.2, mode=text]S[table-format=3.2, mode=text]} 
        \toprule
        \multicolumn{1}{l}{objective} & \multicolumn{1}{c}{FID $\downarrow$} & \multicolumn{1}{c}{IS $\uparrow$} \\
        \midrule
        \cref{eq:p_grad} & 5.78 & 257.02 \\
        \cref{eq:c_grad} & 3.19 & 258.15 \\
        \cref{eq:p_grad,eq:c_grad} & 1.69 & 273.49 \\
        \bottomrule
    \end{tabular}
    \caption{\textbf{Synergy between \cref{eq:p_grad,eq:c_grad}.} Together, they achieve performance that neither could attain in isolation.}
    \label{tab:synergy}
\end{wraptable}
\paragraph{Additional empirical results on prediction and correction synergy.}
In addition to presenting the model progress in \cref{fig:synergy}, we list the final performances of training with only prediction \cref{eq:p_grad}, only correction \cref{eq:c_grad}, and both in \cref{tab:synergy}. Specifically, all models are distilled from SiT-XL/2~\citep{ma2024sit}, and we compare them in terms of FID~\citep{heusel2017gans} and Inception Score~\citep{salimans2016improved} at 1.5M iterations (300 epochs). For each method, we select the optimal $\gamma$ based on the FID performances. Recall from \cref{fig:synergy}, at this point in training, although \cref{eq:p_grad} makes progress, its absolute performance still lags far behind the other two configurations because of the significant error accumulation. In contrast, \cref{eq:c_grad} has already suffered from mode collapse, and its diversity continues to deteriorate.

\begin{figure*}[t]
    \centering
    \begin{subfigure}[t]{0.24\linewidth}
        \centering
        \includegraphics[width=\linewidth]{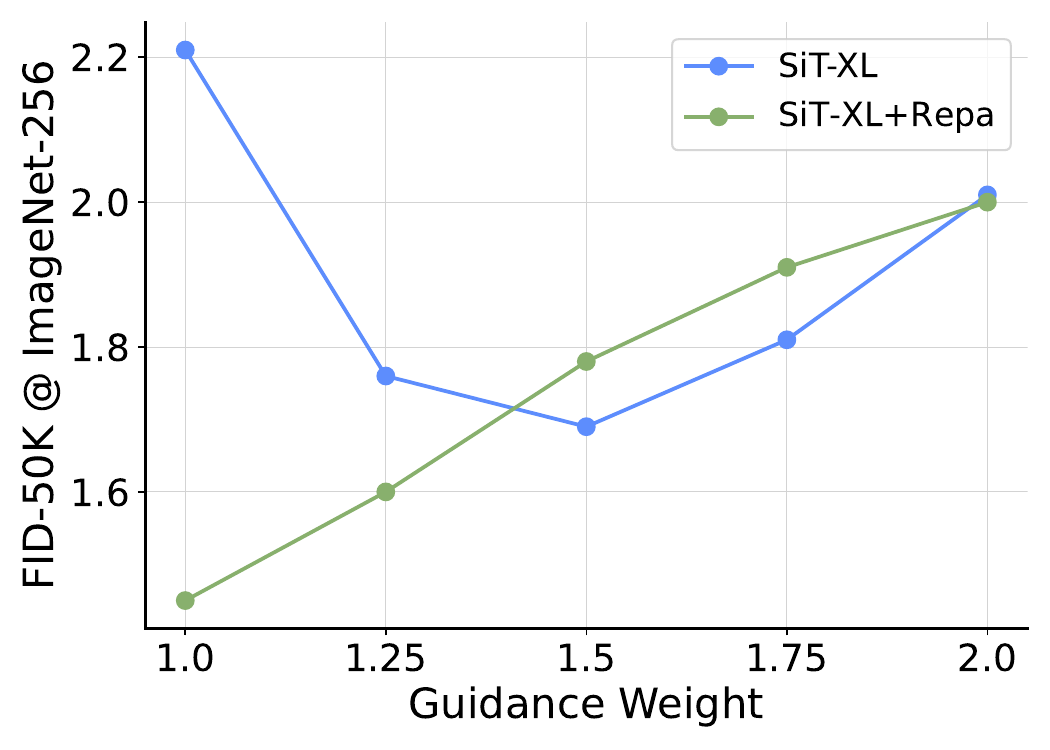}
    \end{subfigure}
    \hfill
    \begin{subfigure}[t]{0.24\linewidth}
        \centering
        \includegraphics[width=\linewidth]{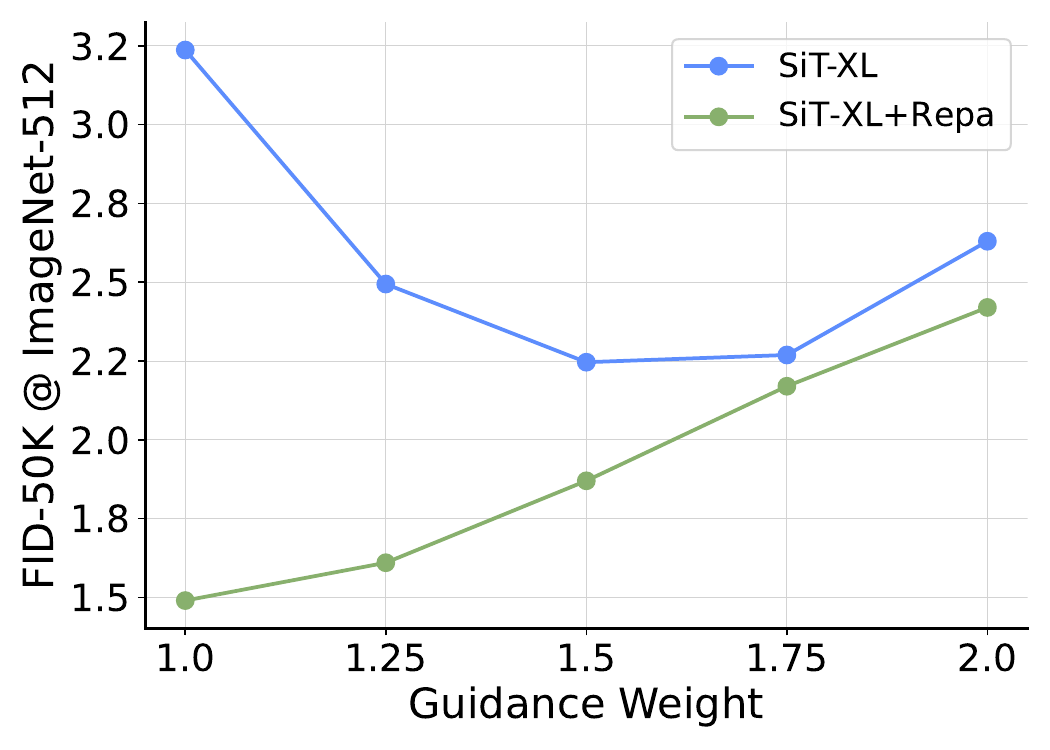}
    \end{subfigure}
    \hfill
    \begin{subfigure}[t]{0.24\linewidth}
        \centering
        \includegraphics[width=\linewidth]{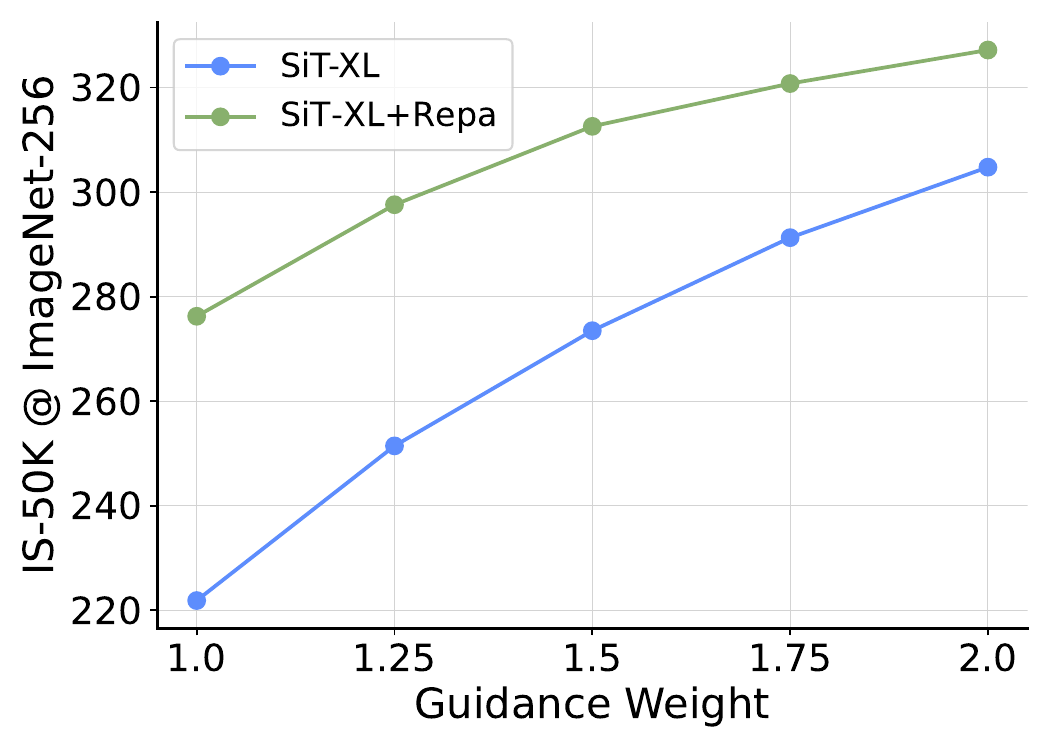}
    \end{subfigure}
    \hfill
    \begin{subfigure}[t]{0.24\linewidth}
        \centering
        \includegraphics[width=\linewidth]{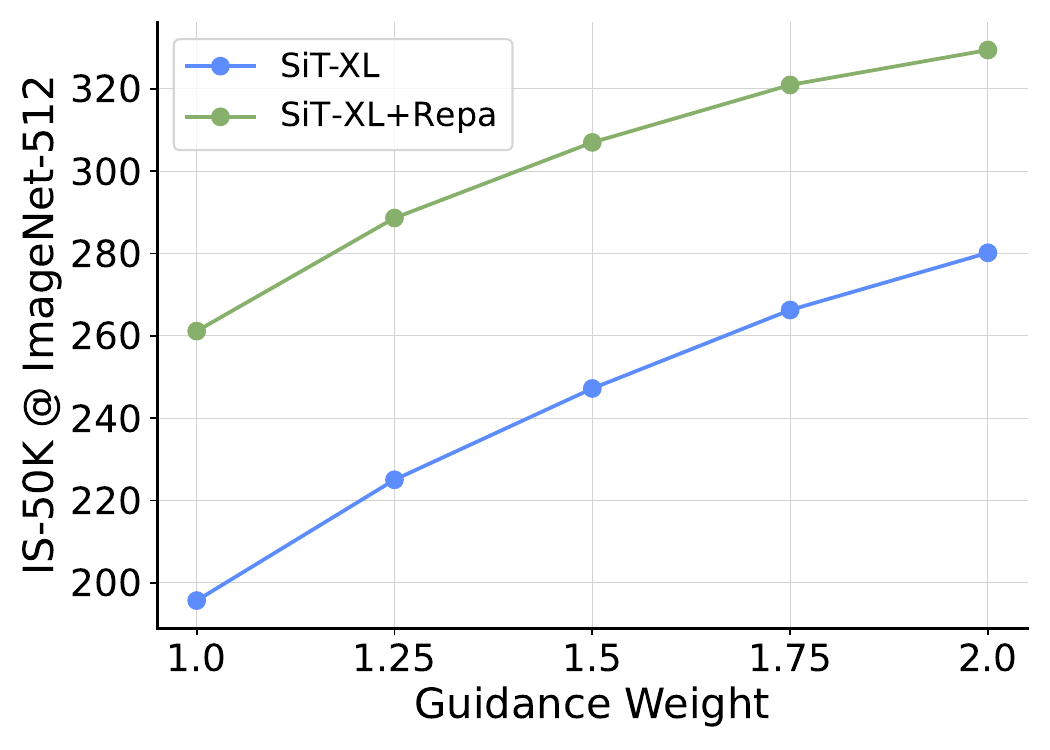}
    \end{subfigure}
    \caption{Performances across different guidance strength parameter $\gamma$ in terms of FID ($\downarrow$) and Inception Score ($\uparrow$).}
    \label{fig:guid}
\end{figure*}

\section{Implementation Details}
\label{sec:implement}

\paragraph{Training.}
We follow the standard practice and train our models in the latent space of the VAE used in~\citet{rombach2022high}. The model architecture we use is based on a standard DiT~\citep{peebles2023scalable} with $2{\times}2$ patches. Recall that the standard input to a flow model for ImageNet is $\vf_\theta(\vx_t, t, c)$ ($c$ is the class label), and we need to include two additional scalar inputs: jump duration $\delta$ and guidance strength $\gamma$. That is, during both training and inference, the model follows $\vf_\theta(\vz, 1, c, \delta, \gamma)$. Thus, we add a few layers to handle these additional conditions. For both, we follow the standard design for including scalar input. Specifically, we use a 256-dimensional frequency embedding~\citep{vaswani2017attention,dhariwal2021diffusion} followed by a two-layer SiLU-activated MLP with the same dimensionality as the model's hidden size. We then add all four embeddings from $t$, $c$, $\delta$, and $\gamma$ together (the newly added ones, $\delta$ and $\gamma$, are initialized at 0) and feed the sum to each block. The same goes for $\vg_\psi$, where it is modified to take input $\vg_\psi(\vx_t, t, c, \gamma)$ as it needs to track different noising flows from different $\gamma$ of $\vf_\theta$. No further architecture changes are necessary for stable and effective training. For each of our reported entries listed in \cref{tab:main}, we present their implementation details in \cref{tab:config}. All model trainings are done with an internal JAX codebase on TPU.

\begin{table*}[t]
    \small
    \centering
    \caption{
        Detailed experimental configurations of our main results.
    }
    \label{tab:config}
    \setlength{\tabcolsep}{20pt}
    \begin{NiceTabular}{
        l
        c
        c
        c
        c
    }
        \toprule
        \multicolumn{5}{l}{\textit{\textbf{Task}}} \\
        \arrayrulecolor{black!30}\midrule
        teacher model & \multicolumn{2}{c}{SiT-XL/2~\citep{ma2024sit}} & \multicolumn{2}{|c}{SiT-XL/2+REPA~\citep{yu2024representation}} \\
        resolution & 256{$\times$}256 & 512{$\times$}512 & 256{$\times$}256 & 512{$\times$}512 \\
        \arrayrulecolor{black}\midrule
        \multicolumn{5}{l}{\textit{\textbf{General}}} \\
        \arrayrulecolor{black!30}\midrule
        iterations (epochs) & 1.5M (300) & 1M (200) & 1.5M (300) & 1M (200) \\
        batch size & \multicolumn{4}{c}{256} \\
        optimizer & \multicolumn{4}{c}{Adam~\citep{kingma2014adam}} \\
        optimizer betas & \multicolumn{4}{c}{(0.9, 0.99)} \\
        optimizer eps & \multicolumn{4}{c}{\num{1e-8}} \\
        weight decay & \multicolumn{4}{c}{0.0} \\
        dropout & \multicolumn{4}{c}{0.0} \\
        $\vf_\theta$ learning rate & \multicolumn{4}{c}{constant \num{3e-5}} \\
        $\vg_\psi$ learning rate & \multicolumn{4}{c}{constant \num{1e-4}} \\
        EMA decay & \multicolumn{4}{c}{0.9999} \\
        \arrayrulecolor{black}\midrule
        \multicolumn{5}{l}{\textit{\textbf{Network}}} \\
        \arrayrulecolor{black!30}\midrule
        params (M) & \multicolumn{4}{c}{678} \\
        FLOPs (G) & 119 & 525 & 119 & 525 \\
        depth & \multicolumn{4}{c}{28} \\
        hidden dim & \multicolumn{4}{c}{1152} \\
        heads & \multicolumn{4}{c}{16} \\
        patch size & \multicolumn{4}{c}{2{$\times$}2} \\
        change from teacher & \multicolumn{4}{c}{additional input for $\delta$ and $\gamma$ in $\vf_\theta$; additional input for $\gamma$ in $\vg_\psi$} \\
        \arrayrulecolor{black}\midrule
        \multicolumn{5}{l}{\textit{\textbf{Training}}} \\
        \arrayrulecolor{black!30}\midrule
        \multicolumn{5}{l}{\cellcolor{gray!10}\textit{Specific to \cref{eq:p_grad}}} \\
        \arrayrulecolor{gray!10}\midrule
        confident region warmup duration & \multicolumn{4}{c}{10K} \\
        $\delta$ type & \multicolumn{4}{c}{discrete; uniform; N=8} \\
        $\delta$ sampling & \multicolumn{2}{c}{LogitNormal(0, 1)} & \multicolumn{2}{|c}{LogitNormal(-0.4, 1.2)} \\
        \% of $\delta=0$ & \multicolumn{4}{c}{10\%} \\
        $\vu$ type & \multicolumn{4}{c}{Heun solver~\citep{karras2022elucidating}} \\
        $k$ & \multicolumn{4}{c}{1} \\
        \arrayrulecolor{black!30}\midrule
        \multicolumn{5}{l}{\cellcolor{gray!10}\textit{Specific to \cref{eq:c_grad}}} \\
        \arrayrulecolor{gray!10}\midrule
        $r$ sampling & \multicolumn{4}{c}{LogitNormal(0.8, 1.6)} \\
        guidance interval & {[0, 0.4]} & {[0, 0.5]} & {[0, 0.3]} & {[0, 0.3]} \\
        \arrayrulecolor{black!30}\midrule
        \multicolumn{5}{l}{\cellcolor{gray!10}\textit{Relevant to both \cref{eq:p_grad,eq:c_grad}}} \\
        \arrayrulecolor{gray!10}\midrule
        split between \cref{eq:p_grad,eq:c_grad} & \multicolumn{4}{c}{75\% : 25\%} \\
        $\gamma$ range & \multicolumn{4}{c}{{[1, 2]}} \\
        $\alpha_\text{ref}$ & \multicolumn{2}{c}{0.3} & \multicolumn{2}{|c}{0.6} \\
        $\alpha_\text{ref}$ schedule ($T_\text{delay}$, $T_\text{warmup}$, $T_\text{decay}$) & \multicolumn{2}{c}{(10K, 10K, $\infty$)} & \multicolumn{2}{|c}{(10K, 10K, 25K)} \\
        \arrayrulecolor{black}\bottomrule
    \end{NiceTabular}
\end{table*}

\paragraph{Evaluation.} We observe small performance variations between TPU-based FID evaluation and GPU-based FID evaluation (ADM’s TensorFlow evaluation suite~\citep{dhariwal2021diffusion}\footnote{\url{https://github.com/openai/guided-diffusion/tree/main/evaluations}}). To ensure a fair comparison with the baseline methods, we convert all of our models into PyTorch, sample all of our models on GPU, and obtain FID scores using the ADM evaluation suite for reporting the final results of our XL-size models in \cref{tab:main,tab:synergy} and \cref{fig:guid}. Additionally, since our model is trained on a range of guidance strengths $\gamma\sim\gU(1,2)$, we can efficiently sweep for an optimal value during inference. We report the best FID in \cref{tab:main}, and provide the complete performance curves in \cref{fig:guid}.

\section{Related Work}
\label{sec:related}

Among the existing distillation approaches, BOOT~\citep{gu2023boot} stands as the most closely related precursor, sharing the distinct operational characteristic of being data-free. However, our works diverge fundamentally in their conceptual positioning and the identified imperative for removing data. BOOT frames the data-free property primarily as a practical/logistical advantage, emphasizing the benefits of bypassing the storage and privacy burdens associated with massive, proprietary training sets. In contrast, we argue that the exclusion of data is not merely a convenience but a theoretical necessity for ensuring distributional fidelity. We elevate the data-free paradigm from a strategy of efficiency to one of correctness, presenting it as the rigorous solution to the identified Teacher-Data Mismatch. Our proposed method also differs significantly from BOOT and its subsequent improvements, which are not necessarily data-free in nature. In particular, \citet{gu2023boot} focuses on the specific signal-ODE parameterization, which requires a separate loss just to enforce the boundary condition $\vf_\theta(\vz,0)=\vz$. In comparison, our prediction objective stems from the properties of average velocity~\citep{geng2025mean}, which satisfy the boundary condition by design. \citet{tee2024physics} and the Lagrangian objective in \citet{boffi2025flow} consider more general ODE formulations. Their optimization involves costly computations of the gradients over the partial derivatives $\partial_\delta\vf_\theta$, whereas ours does not (the partial derivatives in \cref{eq:p_loss_c} are placed inside the stop-gradient operation). This improved training efficiency also originates from the average velocity perspective and our deduced identity in \cref{eq:F_identity}. Furthermore, we introduce an auxiliary correction objective for the accumulated prediction errors, pushing the model performance beyond the current state-of-the-art. In doing so, we believe our work finally completes the picture, validating the data-free paradigm as a robust and promising foundation for the future of generative model acceleration.

Our contribution sits within a much broader body of literature dedicated to accelerating diffusion and flow models. These techniques generally fall into two categories based on their distillation targets. The first category operates at the \textbf{trajectory level}, attempting to compress the complex ODE integration into fewer steps by directly mimicking the sampling path or its solution operator~\citep{luhman2021knowledge,zheng2023fast}. The foundational work of Progressive Distillation~\citep{salimans2022progressive,meng2023distillation} further established the viability of this direction through an iterative strategy that progressively halves the required sampling steps. This paradigm was significantly expanded by Consistency Models~\citep{song2023consistency,berthelot2023tract,song2024improved,lu2024simplifying}, which enforce a property of self-consistency along the trajectory, allowing the model to map arbitrary intermediate states directly to the data origin. More recent approaches~\citep{geng2024consistency,kim2024consistency,heek2024multistep,lee2024truncated,peng2025flow,boffi2025flow,geng2025mean,sabour2025align,liu2025learning,chen2025pi,liu2023flow,lee2025decoupled,liu2023instaflow,ren2024hyper} have further refined this objective by formulating direct matching conditions between the student’s transport map and the teacher’s vector field. The second category operates at the \textbf{distribution level}~\citep{yin2024one,yin2024improved,zhou2024score,zhou2024adversarial,sauer2024adversarial,sauer2024fast,xu2024ufogen,xie2024distillation,luo2023diff,luo2024one,xu2025one,salimans2024multistep}, where the student is trained to match the teacher's marginal distribution directly, often utilizing adversarial or score-based objectives without strictly adhering to the teacher's specific trajectory. We make contributions in both directions by proposing an efficient algorithm for distilling trajectories without data and elucidating additional design spaces for better distribution matching objectives. Together, our proposed predictor-corrector framework can be seen as combining the strengths of the two categories~\citep{lu2024simplifying,zheng2025large}, achieving superior quality while maintaining desired diversity, all without reliance on external data.

\section{Additional Visual Results}
\label{sec:visual}

In \cref{fig:sample_15,fig:sample_250,fig:sample_425,fig:sample_527,fig:sample_537,fig:sample_555,fig:sample_975,fig:sample_980}, we present additional uncurated samples generated by FreeFlow-XL/2 at 512${\times}$512 resolution with only 1-NFE. Again, we emphasize that, during training, we only make use of the teacher model (SiT-XL/2+REPA~\citep{yu2024representation}), without querying any samples from ImageNet.

\begin{figure*}[t]
    \centering
    \begin{minipage}{0.46\linewidth}
        \centering
        \includegraphics[width=\linewidth]{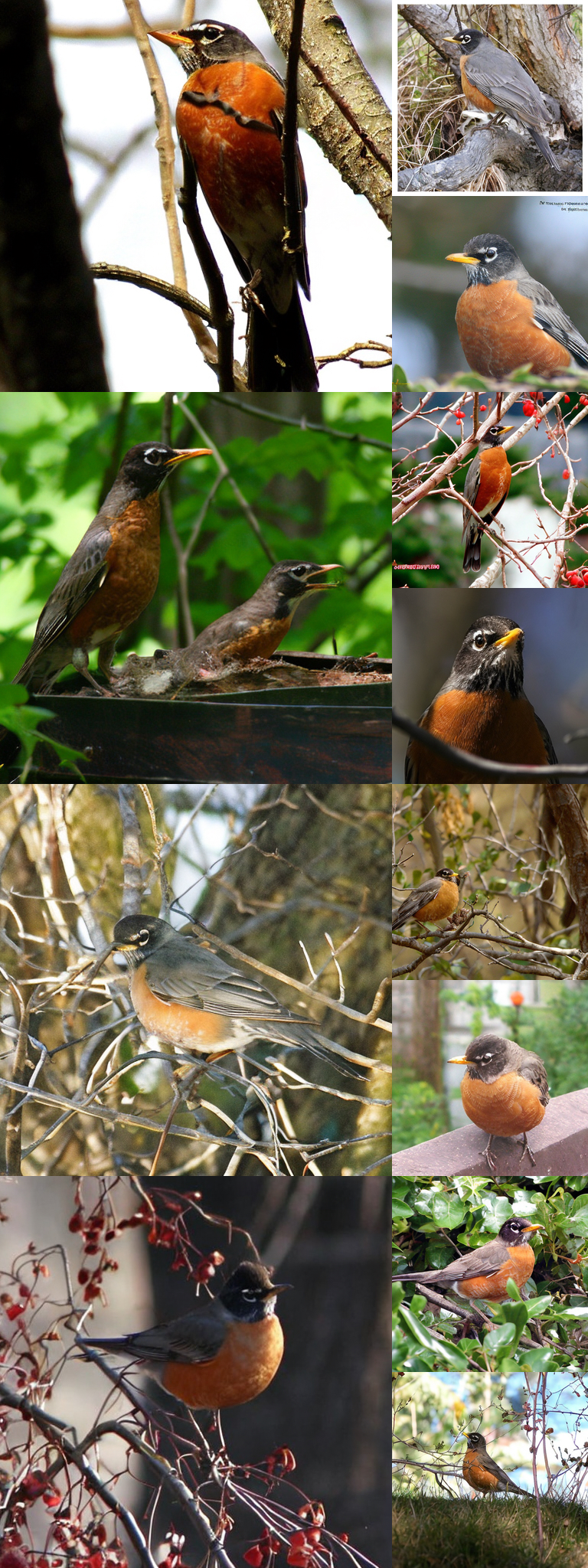}
        \caption{\textbf{Uncurated 512${\times}$512 samples by FreeFlow, 1-NFE.} \\Class label = ``robin'' (15)}
        \label{fig:sample_15}
    \end{minipage}
    \hfill
    \begin{minipage}{0.46\linewidth}
        \centering
        \includegraphics[width=\linewidth]{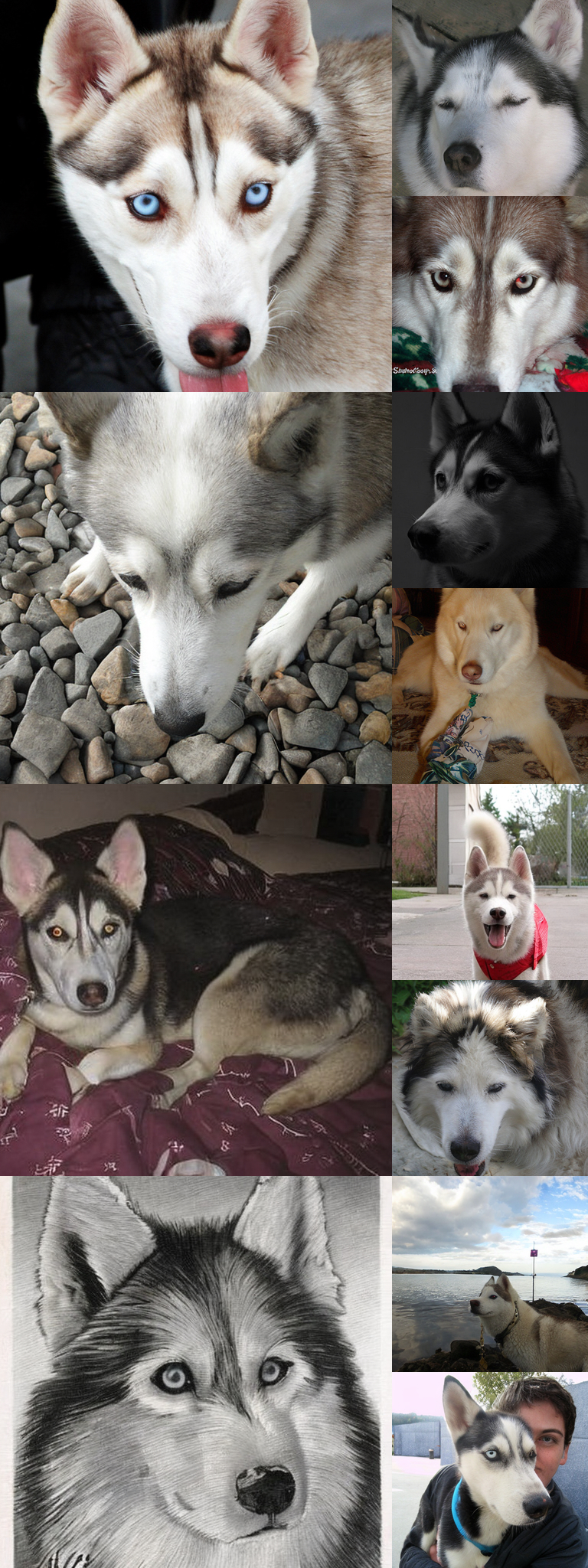}
        \caption{\textbf{Uncurated 512${\times}$512 samples by FreeFlow, 1-NFE.} \\Class label = ``Siberian husky'' (250)}
        \label{fig:sample_250}
    \end{minipage}
\end{figure*}

\begin{figure*}[t]
    \centering
    \begin{minipage}{0.46\linewidth}
        \centering
        \includegraphics[width=\linewidth]{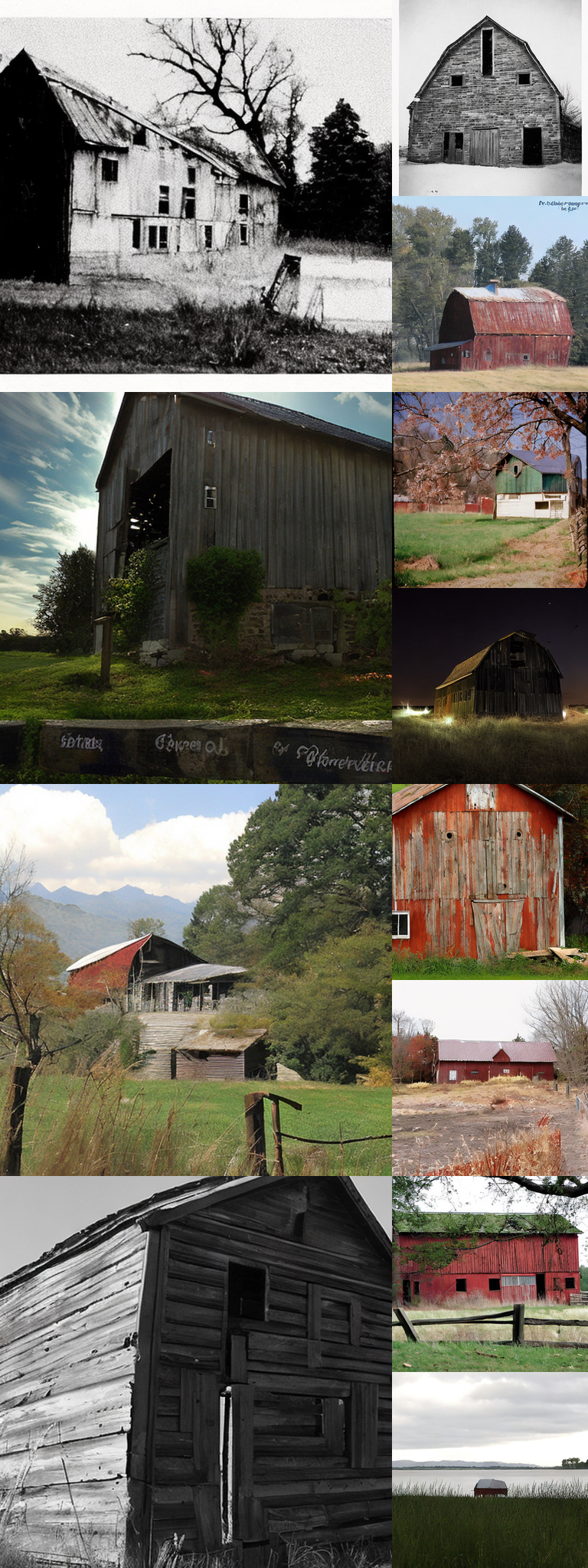}
        \caption{\textbf{Uncurated 512${\times}$512 samples by FreeFlow, 1-NFE.} \\Class label = ``barn'' (425)}
        \label{fig:sample_425}
    \end{minipage}
    \hfill
    \begin{minipage}{0.46\linewidth}
        \centering
        \includegraphics[width=\linewidth]{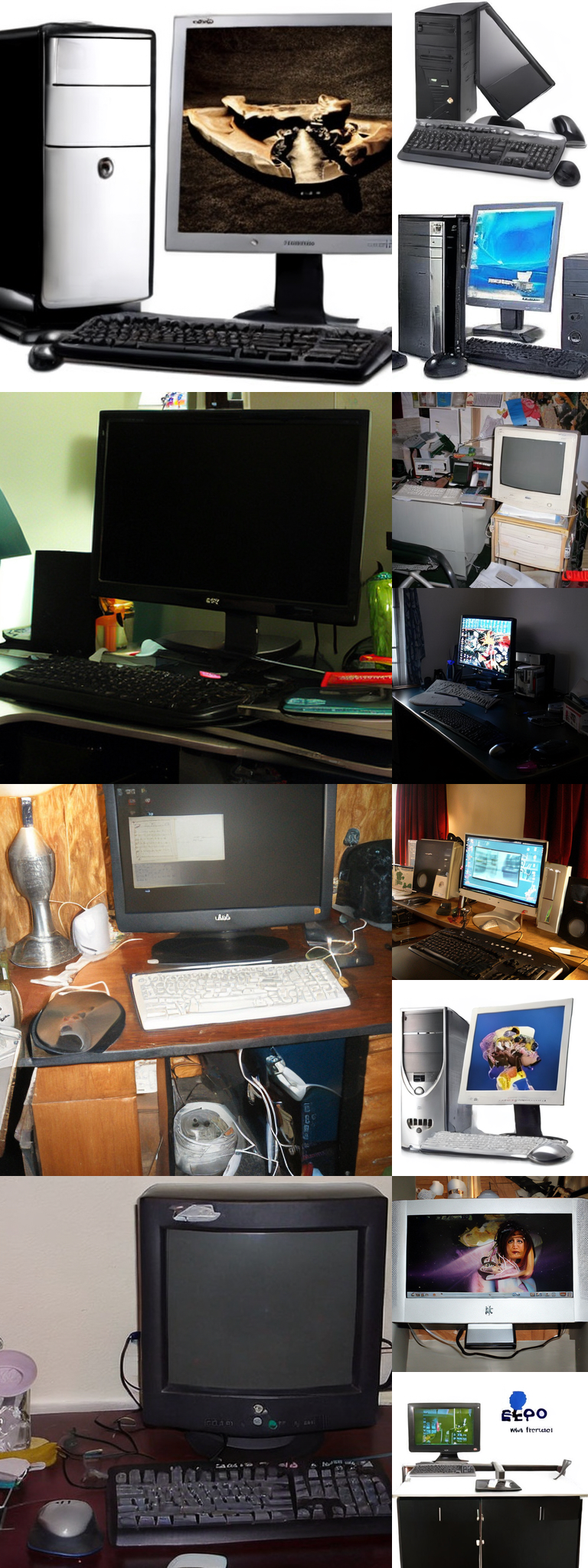}
        \caption{\textbf{Uncurated 512${\times}$512 samples by FreeFlow, 1-NFE.} \\Class label = ``desktop computer'' (527)}
        \label{fig:sample_527}
    \end{minipage}
\end{figure*}

\begin{figure*}[t]
    \centering
    \begin{minipage}{0.46\linewidth}
        \centering
        \includegraphics[width=\linewidth]{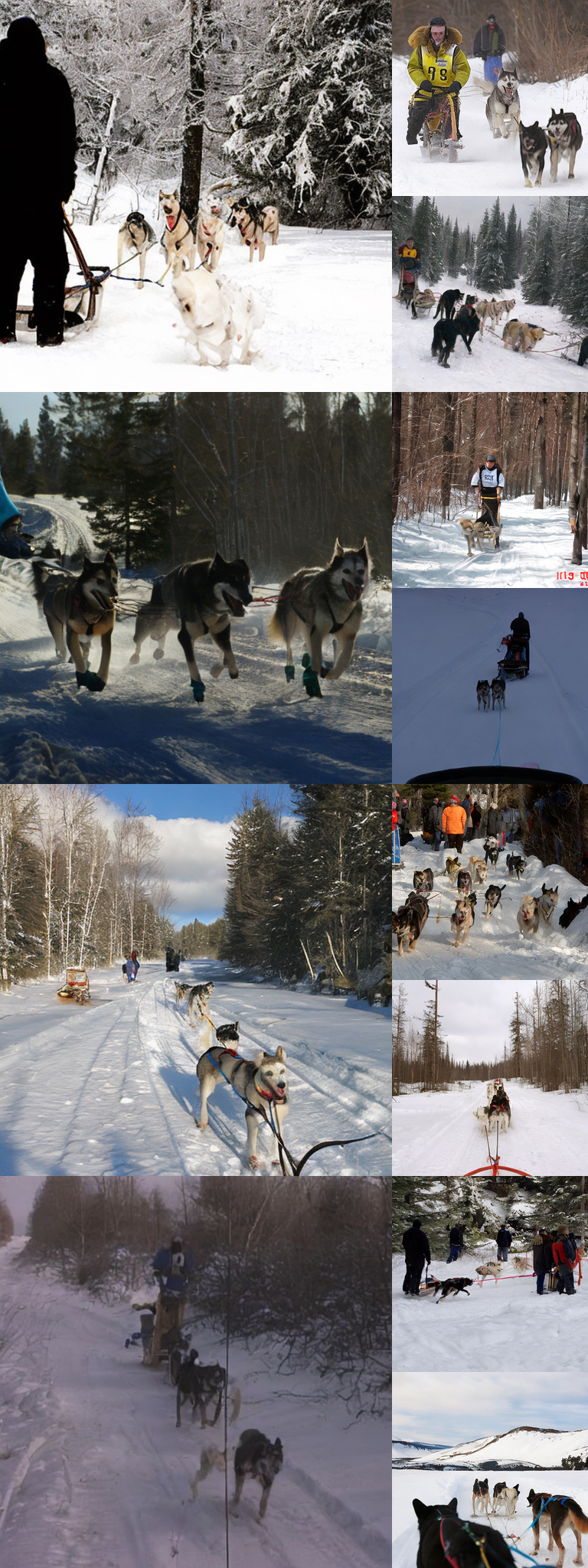}
        \caption{\textbf{Uncurated 512${\times}$512 samples by FreeFlow, 1-NFE.} \\Class label = ``dogsled'' (537)}
        \label{fig:sample_537}
    \end{minipage}
    \hfill
    \begin{minipage}{0.46\linewidth}
        \centering
        \includegraphics[width=\linewidth]{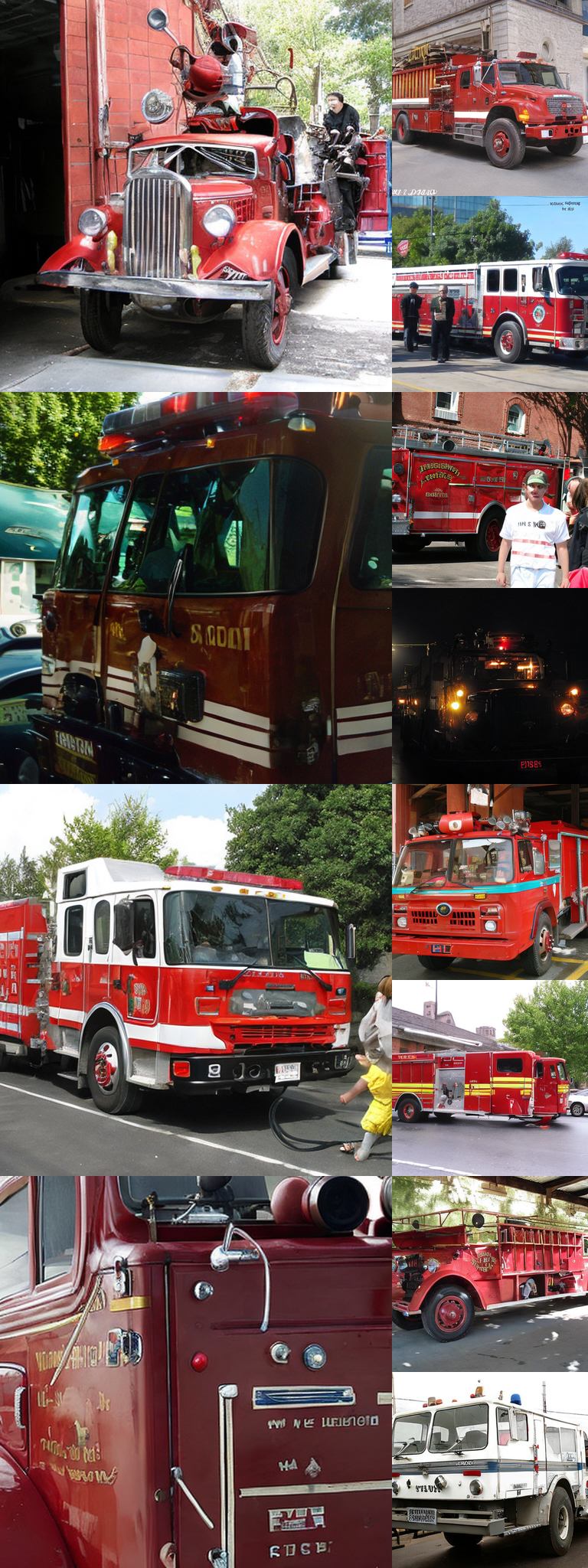}
        \caption{\textbf{Uncurated 512${\times}$512 samples by FreeFlow, 1-NFE.} \\Class label = ``fire truck'' (555)}
        \label{fig:sample_555}
    \end{minipage}
\end{figure*}

\begin{figure*}[t]
    \centering
    \begin{minipage}{0.46\linewidth}
        \centering
        \includegraphics[width=\linewidth]{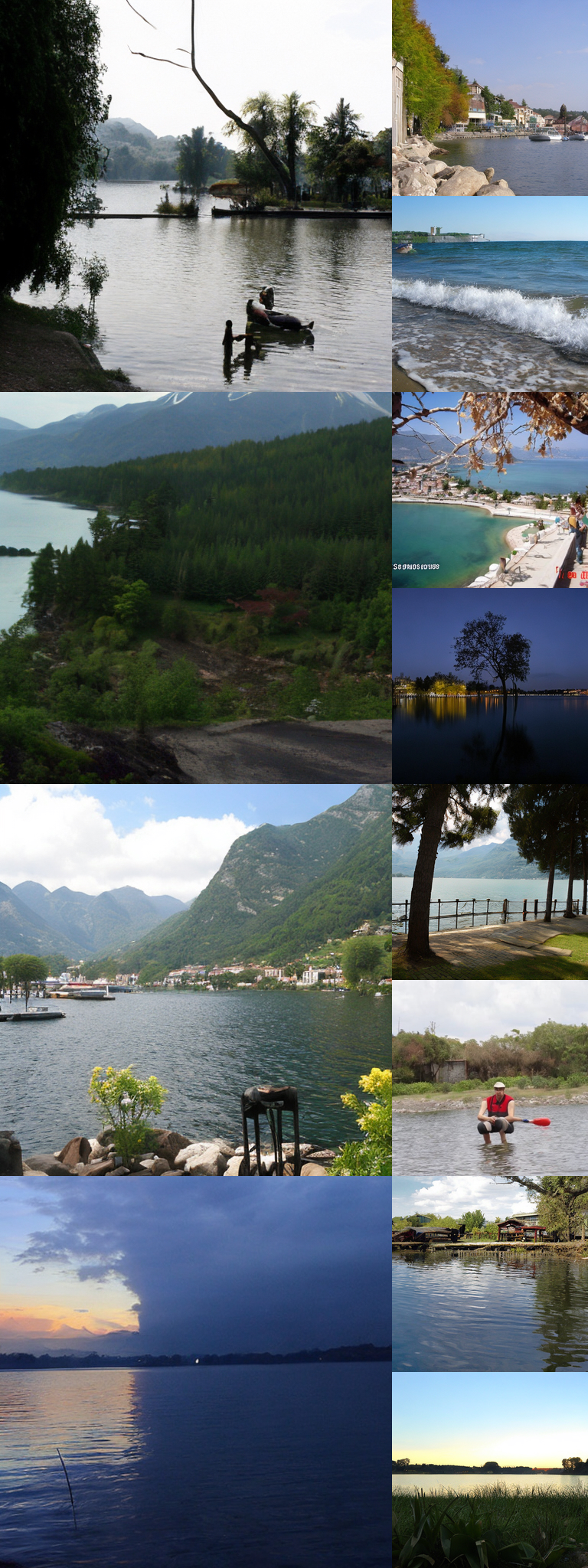}
        \caption{\textbf{Uncurated 512${\times}$512 samples by FreeFlow, 1-NFE.} \\Class label = ``lakeside'' (975)}
        \label{fig:sample_975}
    \end{minipage}
    \hfill
    \begin{minipage}{0.46\linewidth}
        \centering
        \includegraphics[width=\linewidth]{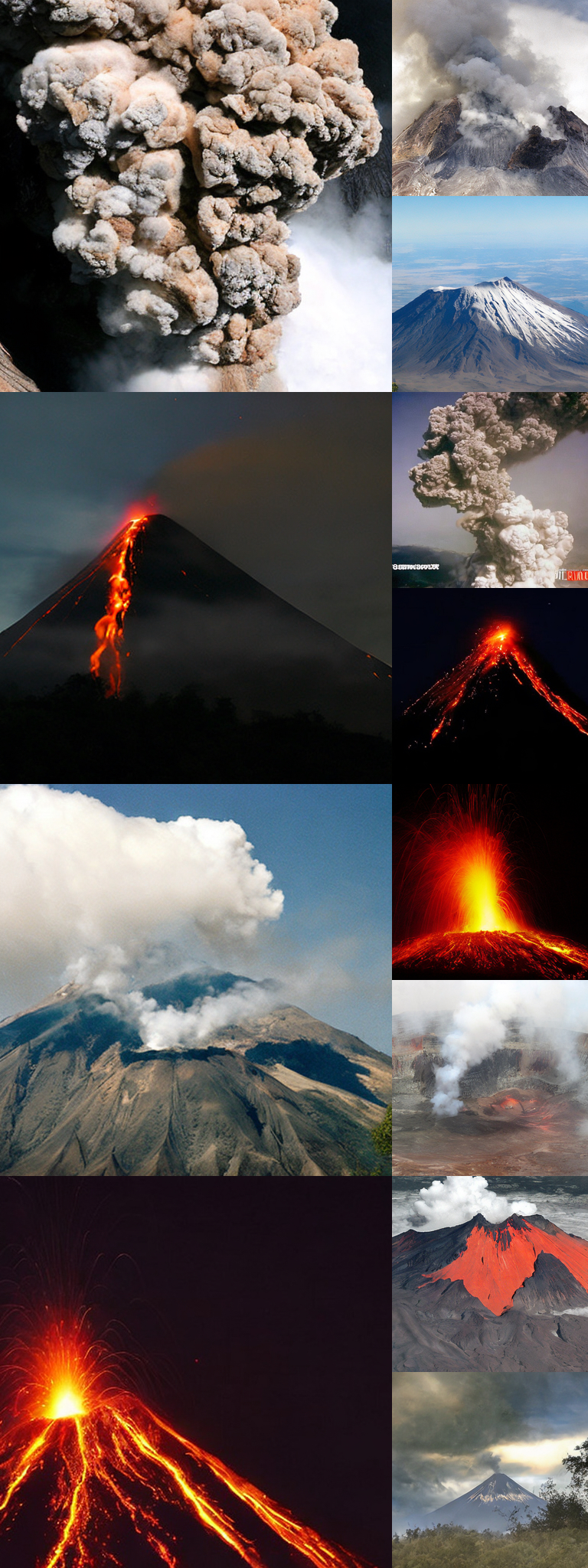}
        \caption{\textbf{Uncurated 512${\times}$512 samples by FreeFlow, 1-NFE.} \\Class label = ``volcano'' (980)}
        \label{fig:sample_980}
    \end{minipage}
\end{figure*}